\documentclass{amsart}

\usepackage[a4paper, margin=1in]{geometry}
\usepackage[T1]{fontenc}

\usepackage{amsmath}
\usepackage{amssymb}
\usepackage{amsfonts}
\usepackage{mathtools}
\usepackage{yhmath}
\usepackage{tabularx}
\usepackage{bbm}
\usepackage{makecell}  
\usepackage{dsfont}

\usepackage{graphicx}
\usepackage{xcolor}
\usepackage{tikz}
\usepackage{pgfplots}
\pgfplotsset{compat=1.18}
\usetikzlibrary{positioning, shapes.geometric, patterns}

\usepackage{booktabs}
\usepackage{multirow}
\usepackage{adjustbox}
\usepackage{subcaption}
\usepackage{caption}
\usepackage{multicol}
\usepackage{enumitem}

\usepackage[ruled,vlined]{algorithm2e}
\usepackage{algpseudocode}

\usepackage{hyperref}
\usepackage{url}

\usepackage{amsthm}
\newtheorem{theorem}[subsection]{Theorem}

\definecolor{darkgreen}{rgb}{0.31, 0.47, 0.26}
\definecolor{matblue}{HTML}{1F77B4}

\renewcommand{\hat}{\widehat}

\newcommand{\Itrain}{\mathcal{I}_{\textrm{train}}}
\newcommand{\Itest}{\mathcal{I}_{\textrm{test}}}

\newcommand{\Ical}{\mathcal I_{\text{cal}}}
\newcommand{\Ieval}{\mathcal I_{\text{eval}}}

\makeatletter
\g@addto@macro{\endabstract}{\@setabstract}
\newcommand{\authorfootnotes}{\renewcommand\thefootnote{\@fnsymbol\c@footnote}}
\makeatother

\title[conformal prediction under adversarial attack]{Game-Theoretic Defenses for Robust Conformal Prediction Against Adversarial Attacks in Medical Imaging}
\date{}

\begin{document}
\maketitle
\begin{center}
\normalsize
\authorfootnotes
Rui Luo\footnote{ruiluo@cityu.edu.hk}\textsuperscript{1}, Jie Bao\footnote{1486103897@qq.com}\textsuperscript{2}, Zhixin Zhou\footnote{zhixin0825@gmail.com}\textsuperscript{3},
Chuangyin Dang\footnote{mecdang@cityu.edu.hk}\textsuperscript{1} \par \bigskip

\textsuperscript{1}City University of Hong Kong \par
\textsuperscript{2}Huaiyin Institute of Technology \par
\textsuperscript{3}Alpha Benito Research \par
\end{center}

\begin{abstract}
Adversarial attacks pose significant threats to the robustness of deep learning models, especially in critical domains such as medical imaging. This paper introduces a framework that integrates conformal prediction with game-theoretic defensive strategies to counter both known and partially unknown adversarial perturbations. We address three research questions: constructing valid and efficient conformal prediction sets under known attacks (RQ1), ensuring coverage under unknown attacks through conservative thresholding (RQ2), and determining optimal defensive strategies within a zero-sum game framework (RQ3). In our methodology, specialized defensive models are trained against specific attack types and aggregated via maximum and minimum classifiers. Extensive experiments on the MedMNIST datasets—PathMNIST, OrganAMNIST, and TissueMNIST—show that our approach preserves high empirical coverage while minimizing prediction set sizes. The game-theoretic analysis indicates that the optimal defense strategy often converges to a singular robust model, outperforming simple uniform approaches. This work advances uncertainty quantification and adversarial robustness, providing a reliable mechanism for deploying deep learning models in adversarial environments.

\end{abstract}
\section{Introduction}

Adversarial attacks \cite{kumari2019harnessing} significantly undermine the robustness of deep learning models, particularly in high-stakes domains such as medical diagnostics \cite{ma2021understanding} and autonomous driving \cite{badjie2024adversarial}. These attacks introduce subtle perturbations to input data that can lead to incorrect and potentially harmful predictions. In critical applications like healthcare and autonomous systems, ensuring both adversarial robustness uncertainty quantification is paramount. Conformal Prediction (CP) \cite{katsios2024multi, meister2023novel, balasubramanian2014conformal} offers a robust framework for uncertainty quantification by generating prediction sets that contain the true label at a predefined confidence level \cite{Vovk2005AlgorithmicLI}.

The interplay between adversarial robustness and conformal prediction \cite{angelopoulos2022conformal, einbinder2022conformal} has received growing attention; however, its implications on the efficiency and robustness of CP sets remain underexplored. Previous studies have focused on adapting scoring functions when either test or calibration data are attacked, but these approaches often suffer from validity issues, undermining the robustness of CP sets under adversarial conditions.

In this work, we propose a novel integration of game-theoretic principles with conformal prediction to construct robust CP sets. Although our empirical evaluation is demonstrated using medical imaging datasets, the underlying framework is applicable to a wide range of domains. Our methodology specifically addresses the following research questions:

\begin{itemize}
    \item How can we construct valid and efficient CP sets under known adversarial attacks?
    \item How can we maintain coverage guarantees in the presence of a set of potential adversarial perturbations?
    \item What are the optimal defensive strategies in a zero-sum game setting where an adversary selects the most potent attack?
\end{itemize}

To tackle these questions, we develop specialized models by partitioning the training data and tailoring them to distinct attack types. We then construct conservative prediction sets based on the highest estimated quantile from calibration data across different adversarial scenarios. This conservative strategy ensures that the CP sets maintain their validity regardless of the attack. Moreover, by formulating the interplay between the defender and the adversary as a zero-sum game, we rigorously analyze the worst-case scenario—namely, the the most adverse adversarial perturbation that maximally impacts prediction set size an adversary can select—and derive optimal mixed defense strategies that minimize prediction set ambiguity while preserving coverage.

By integrating game-theoretic reasoning with CP, our approach enhances the robustness of deep learning systems in adversarial environments. In addition to maintaining high coverage rates, it minimizes prediction set sizes, thereby improving the safety and decision-making quality in critical applications. While our experiments focus on medical imaging, the proposed framework is not limited to this domain and can be extended to other applications where robustness to adversarial perturbations is crucial, such as unmanned driving, financial security analysis and network security protection.

\section{Related Work}
\subsection{Adversarial Attack}
Poisoning attacks represent one of the most direct threats to the training process of machine learning models \cite{tian2022comprehensive}. Current defenses against data poisoning attacks mainly fall into two categories. The first category focuses on anomaly detection using techniques such as nearest neighbors \cite{peri2020deep}, training loss \cite{cretu2008casting}, singular value decomposition \cite{tran2018spectral,diakonikolas2019sever}, clustering \cite{chen2018detecting}, taxonomy \cite{barreno2010security}, or logistic regression \cite{feng2014robust}, as well as other related methods \cite{steinhardt2017certified,paudice2018detection} to filter out anomalies through data-driven approaches. While these countermeasures can mitigate the impact of poisoning to some extent, they exhibit clear shortcomings in terms of effectiveness, cost, and accuracy. The second category focuses on model-driven strategies aimed at enhancing model robustness through techniques such as randomized smoothing \cite{weber2023rab}, ensemble methods \cite{levine2020deep}, data augmentation \cite{ma2019data}, and adversarial training \cite{tao2021better}, among others \cite{liu2017robust}.  
   
However, poisoning attacks are more difficult to detect in deep learning models \cite{barni2019new, saha2020hidden}. Consequently, a wider range of methods has been developed to enhance model robustness against such attacks. For instance, improvements have been made by modifying the model architecture \cite{goldberger2017training}, designing robust loss functions \cite{ghosh2017robust, xu2019l_dmi}, and optimizing loss functions \cite{hendrycks2018using}.

\subsection{Conformal Prediction under Adversarial Attack}
Uncertainty estimation is crucial for the robustness of deep learning models. Conformal Prediction (CP), introduced by Vovk et al. \cite{Vovk2005AlgorithmicLI}, provides distribution-free coverage guarantees but faces challenges when confronted with data poisoning and adversarial attacks. Gendler et al. \cite{gendler2021adversarially} proposed RSCP, enhancing CP's robustness through randomized smoothing by replacing the scoring function to defend against $\ell_2$-norm adversarial perturbations, though its formal guarantees remain limited. Yan et al. \cite{yan2024provably} introduced RSCP+, which adds a hyperparameter $\beta$ and Hoeffding bounds to correct Monte Carlo errors, while incorporating RCT and PTT techniques to improve efficiency. However, these modifications often result in overly conservative prediction sets and increased reliance on the calibration set.
   
In contrast, Ghosh et al. \cite{ghosh2023probabilistically} proposed PRCP, which approaches adversarial perturbations from a probabilistic standpoint by relaxing robustness guarantees under a predefined distribution. Cauchois et al. \cite{cauchois2024robust} addressed distributional shifts using an f-divergence constraint, though the resulting optimization is overly conservative \cite{dvijotham2020framework}. Angelopoulos et al. \cite{angelopoulos2022conformal} extended CP to control expected loss but did not provide algorithms for computing worst-case adversarial loss. Einbinder et al. \cite{einbinder2022conformal} demonstrated that standard CP is robust to random label noise, though adversarial label perturbations were not explored. Furthermore, traditional defenses based on anomaly detection, loss function adjustments, and data clustering \cite{peri2020deep,cretu2008casting,tran2018spectral,chen2018detecting,steinhardt2017certified} struggle to safeguard CP’s unique properties, especially against sophisticated poisoning attacks targeting uncertainty mechanisms. Additionally, Ennadir et al. \cite{ennadir2023conformalized}   investigated the application of CP in adversarial attacks on graph neural networks.
   
Building on these limitations, Liu et al. \cite{liu2024pitfalls} proposed an uncertainty reduction-based adversarial training method (AT-UR), which combines Beta-weighted loss and entropy minimization to enable models to maintain robustness while generating larger and more meaningful prediction sets. Jeary et al. \cite{jeary2024verifiably} introduced Verifiably Robust Conformal Prediction (VRCP), which supports perturbations bounded by various norms (including $\ell_1$, $\ell_2$, and $\ell_\infty$) and regression tasks, using recent neural network verification techniques to restore coverage guarantees under adversarial attacks.

On a different front, two distinct types of attacks have been explored in the context of Conformal Prediction (CP): \emph{evasion attacks} and \emph{contamination (or poisoning) attacks}. In evasion attacks, the adversary perturbs test inputs (often under norm constraints) to induce uncertainty or misclassification during inference. For instance, Zargarbashi et al. \cite{zargarbashi2024robust} proposed using the cumulative distribution function (CDF) of smoothed scores to derive tighter upper bounds on worst-case score variations. Their method addresses both continuous and discrete/sparse data by providing guarantees for evasion attacks (which perturb test inputs) as well as for contamination attacks (which perturb calibration data). In contrast, Li et al. \cite{li2024data} developed a new class of black-box data poisoning attacks that specifically target CP by manipulating the prediction uncertainty rather than directly causing misclassifications. They also proposed an optimization framework to defend against these poisoning attacks. Furthermore, Scholten et al. \cite{scholten2024provably} introduced Reliable Prediction Sets (RPS), which aggregate predictions from classifiers trained on different partitions of the training data and calibrated on disjoint subsets, thereby enhancing CP's resilience to contamination attacks.
   
Recent studies have also highlighted CP's resilience to label noise. Penso and Goldberger \cite{penso2024noise} introduced Noise-Robust Conformal Prediction (NR-CP), which estimates noise-free scores from noisy data, constructing smaller and more efficient prediction sets while maintaining coverage guarantees. NR-CP significantly outperforms other methods in noisy environments. Another study \cite{einbinder2022conformal} explored the inherent robustness of CP, showing that even with noisy labels, CP tends to conservatively maintain coverage, although some adversarial noise conditions may compromise its robustness. Both studies emphasize CP’s resilience to label noise, a perturbation type that, while not intentionally malicious, introduces variability similar to adversarial attacks.

\paragraph{Our work's connection with existing works}
Building upon the foundations laid by prior research, such as RSCP \cite{gendler2021adversarially}, PRCP \cite{ghosh2023probabilistically}, and VRCP \cite{jeary2024verifiably}, our work advances the robustness of CP systems in adversarial settings. While previous studies have focused on specific attack vectors targeting either test data \cite{li2024data} or calibration data \cite{zargarbashi2024robust}, our approach considers a comprehensive adversarial model where the attacker can manipulate both aspects simultaneously without the defender's knowledge. By constructing a conservative prediction set based on the maximum estimated quantile from attacked calibration data, we ensure validity under a wider range of attack scenarios, addressing limitations noted in works such as RSCP+ \cite{yan2024provably} and PRCP \cite{ghosh2023probabilistically}, which may be overly conservative or insufficiently robust against certain attacks.

Furthermore, by framing the interaction between attacker and defender as a zero-sum game, akin to methodologies in adversarial training \cite{tao2021better} and robust optimization \cite{patrini2017making}, our work not only identifies optimal attack strategies but also devises corresponding defense mechanisms that preserve CP system integrity. This game-theoretic perspective extends the robustness guarantees provided by prior works \cite{angelopoulos2022conformal,scholten2024provably}, offering a unified and strategic framework to enhance CP resilience against diverse and adaptive adversarial threats. Additionally, our approach synergizes with uncertainty reduction techniques \cite{liu2024pitfalls} and efficient bound derivations \cite{zargarbashi2024robust}, integrating these methodologies into a cohesive strategy that fortifies CP systems beyond the state-of-the-art methods.

In summary, our research complements and extends existing defense strategies by integrating game theory to model adversarial interactions and proposing novel mechanisms to fortify CP systems. This approach not only builds on the strengths of previous studies \cite{ghosh2023probabilistically,jeary2024verifiably,li2024data,zargarbashi2024robust} but also addresses their limitations, providing enhanced robustness for CP in high-stakes, adversarial environments.

\section{Preliminary and Problem Setup}

In this section, we establish the foundational notation and problem setup for evaluating and enhancing the robustness of CP systems against adversarial attacks. By defining the adversarial attack strategies, defensive models, conformal prediction framework, score functions, and metrics, and by formulating key research questions, we set the stage for developing and analyzing robust prediction mechanisms that maintain valid uncertainty estimates in the presence of deliberate perturbations. The notation used throughout the paper is given in Table \ref{tab:notation}.

\begin{table}[h]
    \centering
    \resizebox{0.8\textwidth}{!}{
    \begin{tabularx}{\textwidth}{@{\extracolsep{\fill}} l X}
        \hline
        \textbf{Notation} & \textbf{Description} \\
        \hline
        $\{(x_i, y_i)\}_{i\in \mathcal{I}}$ & Images and labels where $i$ belongs to the index set $\mathcal{I}$. \\
        $\Itrain, \Ical, \Ieval, \Itest$ & The index set for training, calibration, evaluation, and testing. \\
        $f_0$ & The normal (pre-trained) model trained on clean data. \\
        $g_j$, $j = 1, \ldots, m$ & Adversarial attack functions, where $g_j$ maps the original inputs and labels to perturbed inputs, i.e., $g_j: (X, Y, f_0) \rightarrow X'$. \\
        $f_j$, $j = 1, \ldots, m$ & Defensive models trained to defend specific attacks, where $f_j(x_i)$ denotes the defensive model $f_j$'s prediction on input $x_i$. \\
        $\epsilon$ & Maximum allowable perturbation for adversarial attacks. \\
        $\mathcal{L}(f(x_i), y_i)$ & Loss function of model $f$ with respect to input $x_i$ and true label $y_i$. \\
        $\Gamma(x_i)$ & Conformal prediction set for input $x_i$. \\
        $q_{1-\alpha}$, $\alpha \in (0, 1)$ & Quantile threshold for the desired coverage $1 - \alpha$. \\
        $\mathcal{Y}$ & Set of all possible labels. \\
        \hline
    \end{tabularx}}
    \caption{Notation Used Throughout the Paper}
    \label{tab:notation}
\end{table}

\subsection{Adversarial Attacks and Defensive Models}

\subsubsection{Adversarial Attacks}
Adversarial attacks are strategies employed to intentionally perturb input data in a way that deceives machine learning models into making incorrect predictions. We consider the following adversarial attacks:

\paragraph{Fast Gradient Sign Method (FGSM)}
FGSM is one of the earliest and most straightforward adversarial attack techniques introduced by Goodfellow et al. (2014) \cite{goodfellow2014explaining}. It generates adversarial examples by perturbing the input data in the direction of the gradient of the loss function with respect to the input. The perturbation magnitude is controlled by a small factor $\epsilon$, ensuring that the modifications remain imperceptible to human observers. Mathematically, the adversarial example $x'_i$ is computed as:
\[
    x'_i = x_i + \epsilon \cdot \text{sign}\left(\nabla_{x_i} \mathcal{L}(f_0(x_i), y_i)\right)
\]
Here, $x_i$ is the original input sample, and $\nabla_{x_i} \mathcal{L}(f_0(x_i), y_i)$ represents the gradient of the loss with respect to the input $x_i$.

FGSM is computationally efficient due to its single-step nature, making it suitable for rapid adversarial example generation. However, its simplicity can limit its effectiveness against models trained with robust defenses.

\paragraph{Projected Gradient Descent (PGD)}
PGD is an iterative extension of FGSM that applies multiple small perturbations, projecting the adversarial example back into the allowed perturbation space after each iteration. Introduced by Madry et al. (2017) \cite{madry2017towards}, PGD enhances the effectiveness of the attack by allowing more precise adjustments to the input data through multiple refinement steps. The adversarial example at iteration $t+1$, denoted as $x'^{(t+1)}_i$, is computed as:
\[
    x'^{(t+1)}_i = \Pi_{\mathcal{B}(x_i, \epsilon)} \left\{ x'^{(t)}_i + \alpha \cdot \text{sign}\left(\nabla_{x'^{(t)}_i} \mathcal{L}(f_0(x'^{(t)}_i), y_i)\right) \right\}
\]
In this equation, $\alpha$ is the step size for each perturbation, and $\Pi_{\mathcal{B}(x_i, \epsilon)}$ is the projection operator that ensures $x'^{(t+1)}_i$ remains within the $\epsilon$-ball around the original input $x_i$.

PGD performs adversarial perturbations in multiple iterations, each time adjusting the input based on the gradient of the loss function. This iterative process leads to stronger adversarial examples that are more likely to deceive robust models, making PGD a standard benchmark for evaluating model robustness.

\paragraph{Simultaneous Perturbation Stochastic Approximation (SPSA)}
SPSA \cite{uesato2018adversarial} is a gradient-free adversarial attack that estimates the gradient of the loss function using random perturbations, making it suitable for black-box scenarios where the attacker does not have access to the model's internal parameters. Unlike FGSM and PGD, which rely on exact gradient information, SPSA approximates the gradient by simultaneously perturbing multiple dimensions of the input data with small random noise.

Mathematically, the SPSA attack updates the adversarial example $x'_i$ as follows:
\[
    \hat{\nabla} \mathcal{L}(f_0(x'_i), y_i) \approx \frac{\mathcal{L}(f_0(x'_i + \Delta_i), y_i) - \mathcal{L}(f_0(x'_i - \Delta_i), y_i)}{2\delta} \Delta_i,
\]
\[
    {x'_i}^{(t+1)} = \Pi_{\mathcal{B}(x_i, \epsilon)} \left\{ {x'_i}^{(t)} + \alpha \cdot \text{sign}\left(\hat{\nabla} \mathcal{L}(f_0({x'_i}^{(t)}), y_i)\right) \right\}
\]
Here, $\Delta_i$ is a small random perturbation vector, $\delta$ is a smoothing parameter, and $\alpha$ is the step size. The projection operator $\Pi_{\mathcal{B}(x_i, \epsilon)}$ ensures that the adversarial example remains within the permissible perturbation space.

SPSA is particularly effective in black-box settings where gradient information is not readily available, as it requires only the evaluation of the loss function at perturbed inputs to approximate the gradient. This makes SPSA a versatile and powerful attack method against models with limited attack surface information.

\paragraph{Carlini $\&$ Wagner (CW) Attack}
The CW attack \cite{carlini2017towards} is an optimization-based method for generating minimally perturbed adversarial examples. For the $L_2$ norm, the attack solves
\[
\min_{\delta}\;\|\delta\|_2^2 + c \cdot f(x_{i}+\delta),
\]
subject to
\[
x_{i}+\delta \in [0,1]^n,
\]
where $c>0$ balances the perturbation size with inducing misclassification, and $f$ is a loss function that encourages the attack. A typical definition for a targeted attack is
\[
f(x_{i}') = \max\Big\{ \max_{j \neq t}\{Z(x_{i}')_j\} - Z(x_{i}')_t, -\kappa \Big\},
\]
with $t$ as the target class, $Z(x_{i}')$ representing the model logits, and $\kappa$ controlling the confidence level. In practice, a change-of-variable such as
\[
x_{i}' = \frac{1}{2}\Big(\tanh(w) + 1\Big)
\]
is often adopted to enforce the input box constraint, where $w$ is the unconstrained variable optimized during the attack.

The CW attack is renowned for its effectiveness in producing adversarial examples that remain minimally perturbed while successfully causing misclassification, even on models that have been adversarially trained. Its optimization-based framework often makes it more challenging to defend against compared to simpler gradient-based methods.

\paragraph{Summary and Taxonomy of Attacks}
The aforementioned adversarial attacks can be broadly categorized into several groups:
\begin{itemize}
    \item \textbf{Gradient-based attacks:} Including FGSM and PGD, these methods rely on the gradient information of the loss function. FGSM provides a one-step approximation, while PGD adopts an iterative procedure for improved attack strength.
    \item \textbf{Gradient-free attacks:} Exemplified by SPSA, these approaches estimate gradients through random perturbations, making them suitable for scenarios where the model’s internal parameters are inaccessible.
    \item \textbf{Optimization-based attacks:} Such as the CW attack, which formulates the attack as an optimization problem to achieve minimally perturbed yet effective adversarial examples.
\end{itemize}
In practice, an adversary may even consider using clean (i.e., unperturbed) data or selectively choose one of these attack methods depending on the defensive measures in place, thereby providing a comprehensive evaluation of model robustness against various adversarial perturbations. This selection process, where the adversary chooses an attack method that maximizes its impact on the prediction outcome (i.e., enlarging the prediction set) against a specific defense model, forms a core component of our game-theoretic analysis. Here, the attacker's ability to choose among diverse attack strategies is explicitly modeled, reflecting the adaptive nature of adversarial threats.
\subsubsection{Defensive Models}
Defensive models are pre-trained neural network models designed to withstand specific adversarial attacks. In this framework, we consider:
\paragraph{Normal Model ($f_0$)} A standard model trained solely on clean (non-adversarial) data.
\paragraph{Attack-Specific Models ($f_j$)}: Separate models, each adversarially trained against a specific type of attack $g_j$ (e.g., FGSM, PGD, SPSA).

Each defensive model $f_j$ has been pre-trained and saved prior to experimentation. These models serve as the foundation for constructing a robust prediction mechanism through weighted combinations.

\subsection{Conformal Prediction Framework}

\textbf{Conformal Prediction (CP)} is a statistical framework that provides valid measures of uncertainty for machine learning predictions in a distribution-free manner. We start by assuming that a classification algorithm provides $\hat{p}_y(x)$, which approximates $P(Y=y|X=x)$. While our method and theoretical analysis do not depend on the accuracy of this approximation, it is beneficial to assume that higher values of $\hat{p}_y(x)$ indicate a greater likelihood of sample with feature $x$ having label $y$. We conduct the training procedure on a set $\Itrain$, which is separate from the calibration and test processes. Details will be provided in Section \ref{sec:result}. This separation ensures that $\hat{p}_y(x)$ remains independent of the conformal prediction procedure discussed in this paper.

Let us first consider a single \textit{non-conformity score function} $s(x, y)$. This function is defined such that smaller values of $s(x, y)$ indicate a higher priority of believing $x$ has label $y$. A common choice for the non-conformity score is $s(x, y) = -\hat{p}_y(x)$~\cite{Sadinle_2018}. In this context, conformal prediction for classification problems can be described by the following algorithm. To summarize, we first find a threshold in a set of labeled data, so that $s(x_i, y_i) \leq q_{1-\alpha}$ holds for at least $1-\alpha$ proportion in the set $\Itrain$. Then we use this threshold to define the prediction set for any $x_i$ in the test set, which is the lower level set of the function $y \mapsto s(x_i, y)$.

\begin{algorithm}[t]
\caption{Split Conformal Prediction}
\begin{algorithmic}[1]
\Require Data $\{(x_i, y_i)\}_{i\in\mathcal{I}}$, $\{x_i\}_{i\in\Itest}$, pre-determined coverage probability $1 - \alpha$
\Ensure A prediction set $\Gamma(x_i)$ for each $i \in \Itest$
\State Randomly split $\mathcal{I}$ into $\Itrain$ and $\Ical$.
\State Train a model $\hat{p}_y(x)$ on $\{(x_i, y_i)\}_{i\in \Itrain}$.
\State $q_{1-\alpha} \gets$ the $\lceil (1+|\Ical|)(1-\alpha) \rceil$-th smallest non-conformity score $s(x_i,y_i)$ for $i\in\Ical$. \label{step:two:simple}%
\State For each $x_i \in \Itest$, set $\Gamma(x_i) \gets \{ y \in \mathcal{Y} \mid s(x_i, y) \leq q_{1-\alpha} \}$. \label{step:three:simple}
\end{algorithmic}\label{alg:split}
\end{algorithm}
Algorithm \ref{alg:split} constructs prediction sets $\Gamma(x_i)$ for each $i\in \Itest$, based on the non-conformity scores and a threshold determined by the desired coverage probability $1-\alpha$. The algorithm splits the training data into two subsets: $\Itrain$ for calculating the non-conformity scores, and $\Ical$ for calibration. The threshold $q_{1-\alpha}$ is chosen to ensure the desired coverage probability.

A key property of this method is that, under the assumption of exchangeability of the data points (a weaker assumption than i.i.d.), it guarantees that the resulting prediction sets will contain the true label with probability at least $1-\alpha$:
\begin{align}\label{eq:coverage}
    \mathbb{P}\left(y_i \in \Gamma(x_i)\right) \geq 1 - \alpha.
\end{align}
This property holds regardless of the accuracy of the underlying classification algorithm, making conformal prediction a powerful tool for uncertainty quantification.

\subsection{Problem Formulation}
We aim to address the following research questions to construct a robust conformal prediction system under adversarial attack scenarios:

\begin{enumerate}
    \item \textbf{RQ1}: Given that a known attack $g_j$ is applied to the test dataset $\{x_i\}_{i\in \Itest}$, how can we use a defensive model $f_j$ to construct a valid and efficient conformal prediction set $\Gamma(x_i)$?
    
    \item \textbf{RQ2}: Given that an unknown adversarial attack is applied to the test dataset $\{x_i\}_{i\in \Itest}$, how can we use a defensive model to construct a valid and efficient conformal prediction set $\Gamma(x_i)$?
    
    \item \textbf{RQ3}: Given that an unknown and potentially adversarial attack is applied to the test dataset $\{x_i\}_{i\in \Itest}$, how can we determine the optimal defensive strategy for the defender under a game-theoretic setting where the defender and adversarial attacker comprise a zero-sum game?
\end{enumerate}

To formalize these questions, consider the following setup. Let \( f_0 \) be the model trained on a clean training dataset \(\{(x_i, y_i)\}_{i \in \Itrain}\). An adversary leverages the available training labels to build an adversarial model and then applies an attack function \( g_j \) on the test inputs \(\{x_i\}_{i \in \Itest}\) (with the desired outcome known to the adversary) to generate perturbed inputs:
$x'_i = g_j(x_i, y_i, f_0).$ Note that while the notation includes \( y_i \) for clarity in describing the mechanism, it is not necessary for the adversary to have access to the true labels of the test dataset. The objects from the test dataset, in conjunction with the adversary’s target outcome, are sufficient for carrying out the attack. Defensive models \(\{f_j\}\) are then trained on the modified training set:
$\{(x'_i, y_i)\}_{i\in \Itrain}$, so that each \( f_j \) is tailored to resist the corresponding attack \( g_j \).

The conformal predictor aims to construct prediction sets $\Gamma(x'_i)$ for $i\in \Itest$ that satisfy the coverage guarantee (\ref{eq:coverage}):
\[
\mathbb{P}\left(y_i \in \Gamma(x'_i)\right) \geq 1 - \alpha,
\]
under the adversarial perturbations introduced by $g_j$.

Addressing these research questions is essential for enhancing the robustness of conformal prediction in adversarial settings. RQ1 and RQ2 focus on constructing valid and efficient prediction sets under known and unknown attacks, respectively, ensuring reliable uncertainty quantification. RQ3 explores optimal defensive strategies within a game-theoretic framework, providing strategic insights to mitigate adversarial threats. Together, these investigations advance the development of resilient conformal prediction systems capable of maintaining coverage guarantees even in the presence of adversarial perturbations.

\section{Methodology}
\label{sec:result}
In this section, we detail our approach to addressing the three key research questions outlined in the problem formulation. Each research question is tackled in its respective subsection, where we describe the strategies that form the backbone of our robust CP system against adversarial attacks.

\subsection{Addressing RQ1: Known Adversarial Attack}
\label{sec:method_rq1}

\textbf{RQ1}: Given that a known attack $g_j$ is applied to the test dataset $\{x_i\}_{i\in \Itest}$, how can we use a defensive model $f_j$ to construct a valid and efficient conformal prediction set?

To address RQ1, we leverage the knowledge of the attack \( g_j \) by applying it to the calibration set \(\{(x_i, y_i)\}_{i \in \Ical}\), which is exchangeable with the test set. By replicating the attack \( g_j \) on the calibration data and obtaining \(\{(x'_i, y_i)\}_{i \in \Ical}\), we can compute the corresponding threshold \(q^{j}_{1-\alpha} \). This threshold is then used to construct the prediction sets \(\Gamma(x'_i)\) for each \( i \in \Itest \). This approach follows the traditional split conformal prediction framework (Algorithm \ref{alg:split}), with the key difference that both the calibration and test datasets are subjected to adversarial attacks. The exchangeability between the calibration and test sets ensures that the coverage guarantee (\ref{eq:coverage}) of conformal prediction remains valid even under these adversarial perturbations.

\begin{algorithm}[ht]
\small{
\caption{\small{Constructing Conformal Prediction Sets under Known Attack (RQ1)}}
\begin{algorithmic}[1]
\Require Training dataset $\{(x_i, y_i)\}_{i \in \Itrain}$, Calibration dataset $\Ical = \{(x_i, y_i)\}_{i \in \Ical}$, 

Test dataset $\Itest = \{x_i\}_{i \in \Itest}$, Known attack function $g_j$, 

Coverage probability $1 - \alpha$
\Ensure Prediction sets $\left( \Gamma(x_i) \right)_{i \in \Itest}$

\rlap{\hbox{\textcolor{darkgreen}{\Comment{Train defensive models $f_k$ for each attack using attacked training set:}}}}

\State \For{each attack function $g_k, k\in 1, \dots, m$}{Apply attack $g_k$ to the training set to obtain:

\centerline{$x'_i = g_k(x_i, y_i, f_0), i\in \Itrain$.}

Train the defensive model $f_k$ on the attacked training set $\{(x'_i, y_i)\}_{i\in \Itrain}$.}

\rlap{\hbox{\textcolor{darkgreen}{\Comment{Calibrate each defensive model by applying the known attack $g_j$ to the calibration set:}}}}
\State Apply the known attack $g_j$ to the calibration set $\Ical$ to obtain $ \{(x'_i, y_i)\}_{i \in \Ical}$, where $x'_i = g_j(x_i, y_i, f_0)$.

\State Compute non-conformity scores $s(x_i, y_i)$ based on $f_j(x_i)$ for each $(x_i, y_i) \in \Ical$.
\State Determine the quantile threshold $q_{\alpha}$ as the $\lceil (1 + |\Ical|)(1 - \alpha) \rceil$-th smallest score in $\Ical$.

\rlap{\hbox{\textcolor{darkgreen}{\Comment{Construct prediction sets using the quantile estimated on the attacked calibration set}}}}
\State \For{each test instance $x_i$, $i \in \Itest$}{
Construct the prediction set: \[
\Gamma(x_i) = \{ y \in \mathcal{Y} \mid s(x_i, y) \leq q_{\alpha} \}
\]}
\end{algorithmic}\label{alg:rq1}}
\end{algorithm}

\textbf{Procedure}: The steps outlined in Algorithm~\ref{alg:rq1} systematically apply the known adversarial attack to the calibration set, train an adversarially robust defensive model $f_j$, compute non-conformity scores based on the defensive model's predictions, and determine the appropriate quantile threshold to construct valid and efficient prediction sets for the test data.

\textbf{Intuition}: By applying the known attack $g_j$ to the calibration set and training the defensive model $f_j$ accordingly, we create a scenario where the calibration set remains exchangeable with the adversarially perturbed test set. This alignment guarantees that the coverage condition $\mathbb{P}(y_i \in \Gamma(x_i)) \geq 1 - \alpha$ holds, while the adversarially trained classifier $f_j$ optimizes the efficiency of the prediction sets.

\subsection{Addressing RQ2: Unknown Adversarial Attack}
\label{sec:method_rq2}

\textbf{Research Question 2 (RQ2)}: Given that an unknown adversarial attack is applied to the test dataset $\Itest$, how can we use a defensive model to construct a valid and efficient conformal prediction set $\Gamma(x_i)$?

In scenarios where the adversarial attack is unknown, we cannot rely on applying the same attack strategy to the calibration set to maintain exchangeability. To ensure the coverage condition under these circumstances, we adopt a conservative approach by considering the worst-case adversarial scenario.

\begin{algorithm}[ht]
\small{
\caption{Constructing Conformal Prediction Sets Across All Potential Attacks (RQ2)}
\begin{algorithmic}[1]
\Require Training dataset $\{(x_i, y_i)\}_{i \in \Itrain}$, Calibration dataset $\{(x_i, y_i)\}_{i \in \Ical}$, 

Post-unknown attack test dataset $\{x_i\}_{i \in \Itest}$, 

Set of potential attack functions $g_j, j=1,\dots,m$, 

Coverage probability $1 - \alpha$.
\Ensure Prediction sets $\left( \Gamma(x_i) \right)_{i \in \Itest}$ for a single defensive model $f_k$
\State Based on the training method for $f_k$ outlined in Algorithm \ref{alg:rq1}, we obtain the pre-trained defensive models $f_k$, $k=1, \dots, n$.

\rlap{\hbox{\textcolor{darkgreen}{\Comment{Determine the largest quantile score across various attacks $g_j$ on the calibration set:}}}}
\State 
\For{each attack function $g_j, j\in 1, \dots, m$}
{Apply attack $g_j$ to the calibration set to obtain:

\centerline{$x'_i = g_j(x_i, y_i, f_k), i\in \Ical$.}

Use $f_k(x'_i)$ to compute non-conformity scores $\{s^j_{k}(x'_i, y_i)\}_{i\in \Ical}$.

Compute the quantile threshold as:

\centerline{$q^{k_{j}}_{1-\alpha}$ as the $\lceil (1 + |\Ical|)(1 - \alpha) \rceil$-th smallest score in ${s^j_{k}(x'_i, y_i)}_{i\in \Ical}$.}}

\State Determine the maximum quantile threshold $q^{k}_{1-\alpha} = \max_{j=1, \dots, m} q^{k_{j}}_{1-\alpha}$.

\rlap{\hbox{\textcolor{darkgreen}{\Comment{Construct conservative prediction sets using the largest quantile score to ensure coverage:}}}}
\State 
\For{each test instance $x_i$, $i \in \Itest$}
    {Construct the conservative prediction set for the defensive model $f_k(x'_i)$:
    \[
    \Gamma(x_i) = \{ y \in \mathcal{Y} \mid s(x_i, y) \leq q^{k}_{1-\alpha} \}
    \]}
\end{algorithmic}\label{alg:rq2}}
\end{algorithm}

\textbf{Procedure}: The steps in Algorithm~\ref{alg:rq2} involve applying each potential adversarial attack from the set $\mathcal{G}$ to the calibration set, training defensive models against each attack, computing corresponding quantile thresholds, and selecting the maximum quantile threshold to ensure coverage under the worst-case scenario when the exact nature of the attack is unknown.

\textbf{Intuition}: By evaluating multiple potential adversarial attacks and selecting the most stringent quantile threshold, we ensure that the prediction sets remain valid even when the exact nature of the attack is unknown. This conservative approach guarantees coverage but may result in larger prediction sets due to the maximization step.

\subsection{Addressing RQ3: Game-Theoretic Adversarial Attack}
\label{sec:method_rq3}

\textbf{Research Question 3 (RQ3)}: Given that an unknown and potentially adversarial attack is applied to the test dataset \(\Itest\), how can we determine the optimal defensive strategy for the defender under a game-theoretic setting?

To address RQ3, we model the interactions between the attacker and the defender as a zero-sum game, where the attacker aims to maximize the size of the prediction sets (thereby inducing uncertainty), and the defender seeks to minimize the prediction set sizes while maintaining the coverage guarantee. Specifically, the defender selects a defensive model from a predefined set of defensive models\footnote{We can have more defensive models than attacks, as exemplified by the Maximum classifier and the Minimum classifier introduced in Section \ref{subsec:rq3_experiment}.} \(\{f_k\}_{k=1}^{p}\), while the attacker chooses an attack from a set of possible attacks \(\{g_j\}_{j=1}^{m}\). Formally, the defender's objective is to minimize the maximum possible adversarial impact by solving the following optimization problem:
\[
\min_{f_k, q_{1-\alpha}^{k}} \max_{g_j} \sum_{i \in \Itest} \mathbb{E}\left[ |\Gamma(x'_i)| \right],
\]
subject to the coverage constraint:
\[
\mathbb{P}\left(y_i \in \Gamma(x_i)\right) \geq 1 - \alpha,
\]
where \( x'_i = g_j(x_i, y_i, f_0) \) represents the adversarially perturbed input under attack \( g_j \), and the non-conformity score associated with the perturbed input \( s^{k}(x'_i, y_i) \) is computed based on $f_k$, i.e., the outcome is determined by both parties. Here, \(\Gamma(x'_i)\) is the prediction set constructed using the defensive model \( f_k \) and the corresponding threshold \( q_k \). As illustrated in Algorithm \ref{alg:rq2}, $q_k$ is obtained as the maximum of all $q^{j}_{1-\alpha}$ of the scores computed using the defensive model $f_k$ on the calibration set attacked by $g_j$. To rigorously support our claim that the coverage guarantee is maintained even in the presence of strategic adversarial perturbations, we state and prove the following theorem.

\begin{theorem}[Robust Coverage Guarantee]
\label{thm:robust_coverage}
Let \(\{f_k\}_{k=1}^{p}\) be a set of predefined defensive models and \(\{g_j\}_{j=1}^{m}\) be a set of adversarial attacks. Suppose that:
\begin{enumerate}
    \item The calibration set \(\mathcal{I}_{\text{cal}}\) and the test \(\mathcal{I}_{\text{test}}\) set are exchangeable.  
    \item For any given defensive model \(f_k\) and any attack \(g_j\), the nonconformity score \(s^k(x'_i, y_i)\) satisfies regularity conditions that ensure the existence and proper ordering of quantiles (e.g., measurability and monotonicity properties).
    \item The threshold \(q_{1-\alpha}^{k}\) is defined as
    \[
    q_{1-\alpha}^{k} = \max_{j \in \{1,\ldots,m\}} q^{k_{j}}_{1-\alpha},
    \]
    where each \(q^{k_{j}}_{1-\alpha}\) is computed on \(\mathcal{I}_{\text{cal}}\) so that
    \[
    \frac{1}{|\mathcal{I}_{\text{cal}}|}\sum_{i \in \mathcal{I}_{\text{cal}}} \mathbf{1}\{ s_{k}^{j}(x'_i, y_i) > q^{k_{j}}_{1-\alpha} \} \le \alpha.
    \]
\end{enumerate}
Then, for every attack \(g_j\), the prediction set
\[
\Gamma(x'_i) = \{ y \in \mathcal{Y} : s_{k}^{j}(x'_i, y) \le q_{1-\alpha}^{k} \}
\]
satisfies the coverage guarantee
\[
\mathbb{P}(y \in \Gamma(x'_i)) \ge 1-\alpha.
\]
\end{theorem}

\begin{proof}
Fix any defensive model \(f_k\) and an arbitrary attack \(g_{j}\). By the construction on the calibration set under the attack \(g_{j}\), the quantile \(q^{k_{j}}_{1-\alpha}\) satisfies
\[
\frac{1}{|\mathcal{I}_{\text{cal}}|}\sum_{i \in \mathcal{I}_{\text{cal}}} \mathbf{1}\Big\{ s_{k}^{j}(x'_i, y_i) > q^{k_{j}}_{1-\alpha} \Big\} \le \alpha.
\]
This implies
\[
\mathbb{P}\Big( s_{k}^{j}(x'_i, y_i) \le q^{k_{j}}_{1-\alpha} \Big) \ge 1-\alpha,
\]
where the probability is taken with respect to the data distribution (leveraging the exchangeability or i.i.d. assumption between the calibration and test sets).

Since the threshold \(q^{k}_{1-\alpha}\) is defined as the maximum over all attacks, we have
\[
q^{k}_{1-\alpha} = \max_{1 \le j \le m} q^{k_{j}}_{1-\alpha} \ge q^{k_{j}}_{1-\alpha}.
\]
Under the assumption of the monotonicity of the indicator function, i.e.,
\[
\{s^k(x'_i, y_i) \le q^{k_{j}}_{1-\alpha}\} \subseteq \{s^k(x'_i, y_i) \le q^{k}_{1-\alpha}\},
\]
it directly follows that
\[
\mathbb{P}\Big( s^k(x'_i, y_i) \le q_k \Big) \ge \mathbb{P}\Big( s^k(x'_i, y_i) \le q^{j_0}_{1-\alpha} \Big) \ge 1-\alpha.
\]
Since \(j\) was arbitrary, this argument holds for any attack \(g_j\). Hence, the prediction set constructed with \(q^{k}_{1-\alpha}\) satisfies the desired coverage guarantee:
\[
\mathbb{P}(y \in \Gamma(x'_i)) \ge 1-\alpha.
\]
\end{proof}

This optimization framework seeks to identify the defensive model \( f_k \) and threshold \( q_k \) that offer the most robust defense against the worst-case attack \( g_j \). It further weighs the balance between coverage rate and the size of conformal prediction sets. By framing the interaction between the defender and attacker within a game-theoretic context, the framework systematically assesses and bolsters the resilience of conformal prediction sets against a multitude of adversarial threats. Simultaneously, the defender strives to minimize the size of the prediction set to decrease model uncertainty, thereby ensuring that the predictive system maintains coverage guarantees while offering a smaller, more credible range for medical imaging decisions, as excessively large prediction sets lack practical significance.

\textbf{Procedure}: The steps outlined in Algorithm~\ref{alg:rq3} involve utilizing a calibration set and an evaluation set, which are interchangeable with the test set. The process begins by applying each potential adversarial attack to the calibration set. Subsequently, defensive models are trained against each attack. Quantile thresholds and corresponding prediction set sizes are then computed on the evaluation set to construct the payoff matrix. A key aspect of this algorithm, specifically in Step 4, is the utilization of the NashPy package to solve the zero-sum game. This involves employing an enumerative approach to iterate through all possible support set combinations, constructing and solving the corresponding linear equation systems to identify mixed strategies that satisfy the Nash equilibrium conditions. For each identified strategy combination, NashPy verifies its non-negativity and normalization, ultimately returning all feasible mixed strategies as the Nash equilibria of the game. Finally, the selected quantile threshold is applied to construct conservative prediction sets for the test data.

\textbf{Intuition}: By modeling the interaction between the attacker and the defender as a zero-sum game and evaluating the expected prediction set sizes across different attack and defense strategies, we can identify the defensive strategy that offers the most robust protection against the worst-case adversarial attacks. Selecting the largest quantile threshold $q_k$ for each selected model $f_k$ ensures that the coverage condition is maintained, even in the presence of the most challenging adversarial perturbations.

\begin{algorithm}[H]
\small
\caption{Game-Theoretic Optimal Defensive Strategy (RQ3)}
\begin{algorithmic}[1]
\Require Calibration dataset $\{(x_i, y_i)\}_{i \in \Ical}$, Evaluation dataset $\Ieval = \{(x_i, y_i)\}_{i \in \Ieval}$, 

Test dataset $\{x_i\}_{i \in \Itest}$, Set of potential attack functions $g_j, j=1,\dots,m$, 

Set of defensive models $f_k, k=1,\dots,p$, Coverage probability $1 - \alpha$
\Ensure Optimal defensive strategy for constructing prediction sets $\left(\Gamma(x_i)\right)_{x_i \in \Itest}$

\rlap{\hbox{\textcolor{darkgreen}{\Comment{Determine the largest quantile score $q_k$ for each defensive model $f_k$ across various attacks:}}}}

\State
\For{each defensive model $f_k, k=1,\dots,p$}
{

\For{each attack function $g_j, j= 1, \dots, m$}
{Apply attack $g_j$ to the calibration set to obtain:

\centerline{$x'_i = g_j(x_i, y_i, f_0), i\in \Ical$.}

Use $f_k(x'_i)$ to compute non-conformity scores $\{s_{k}^j(x'_i, y_i)\}_{i\in \Ical}$.

Compute the quantile threshold as:

\centerline{$q^{k_{j}}_{1-\alpha}$ as the $\lceil (1 + |\Ical|)(1 - \alpha) \rceil$-th smallest score in $\{s_{k}^j(x'_i, y_i)\}_{i\in \Ical}$.}}

Determine the maximum quantile threshold $q_{1-\alpha}^{k} = \max_{j=1, \dots, m} q^{k_{j}}_{1-\alpha}$ for defensive model $f_k$.
}

\rlap{\hbox{\textcolor{darkgreen}{\Comment{Estimate the prediction set sizes on the evaluation set:}}}}

\State
\For{each defensive model $f_k, k=1,\dots,p$}
{

\For{each attack function $g_j, j=1,\dots,m$}
{
Compute the average prediction set size on the evaluation set:

\centerline{$\frac{1}{|\Ieval|}\sum_{i \in \Ieval} |\Gamma(x'_i)| $,}

where $\Gamma(x'_i) = |\{ y \in \mathcal{Y} \mid s_{k}^j(x'_i, y) \leq q_{1-\alpha}^{k} \}|$, with $x'_i$ the perturbed input attacked by $g_j$ and $s_{k}^j(x'_i, y)$ computed using $f_k$.

}
}

\rlap{\hbox{\textcolor{darkgreen}{\Comment{Optimal Defensive Strategy Derivation under Nash Equilibrium:}}}}
\State Construct the payoff matrix $P$ where each entry $P_{k,j}$ represents the expected prediction set size when the attacker selects attack $g_j$ and the defender selects defensive model $f_k$.
\State Solve the zero-sum game using the payoff matrix $P$ to identify the optimal defensive strategy.
\State Select defensive models \( f_k, \, k = 1, \dots, p \) according to the optimal defensive strategies.

\rlap{\hbox{\textcolor{darkgreen}{\Comment{Construct the prediction sets using the optimal defensive strategy:}}}}
\State
\For{each test instance $x_i, i \in \Itest$}
    {Construct the conservative prediction set:
    \[
    \Gamma(x_i) = \{ y \in \mathcal{Y} \mid s(x_i, y) \leq q_{1-\alpha}^{k} \}
    \]
    with probability equal to the probability of selecting \( f_k \) in the (potentially mixed) strategy.}
\end{algorithmic}\label{alg:rq3}
\end{algorithm}

\section{Experiment}
\subsection{Experimental Design}
\subsubsection{Dataset Preparation}
The experiments utilize the MedMNIST dataset suite~\cite{yang2023medmnist,yang2021medmnist}. Three specific datasets are employed, each with unique characteristics. The datasets are selected via the \texttt{--dataset} command-line argument, which allows the user to specify one of the available datasets (e.g., \texttt{PathMNIST}, \texttt{OrganAMNIST}, or \texttt{TissuMNIST}). This design provides flexibility in choosing different organ-specific visual recognition tasks and facilitates reproducibility of the experiments.

Specifically, \textbf{PathMNIST} comprises 97,176 images of colon pathology, each of size $28\times28$ pixels, categorized into nine classes. \textbf{OrganAMNIST} contains 34,561 abdominal CT scans, similarly resized to $28\times28$ pixels, and distributed across eleven classes. \textbf{TissuMNIST} includes 236,386 images of human kidney cortex cells originally sized $32\times32\times7$. For TissuMNIST, 2D projections are obtained by taking a maximum pixel value along the axial axis, and the resulting images are resized to $28\times28$ pixels. For the predictive models, all experiments deploy the ResNet18 architecture as the backbone \cite{he2016deep}.

Overall, the experimental setup is designed to rigorously evaluate both standard and adversarially-trained models across multiple datasets and attack scenarios, providing insights into model performance and adversarial resilience.

\subsubsection{Data Splitting and Experimental Setup}
For each dataset and every defensive model, including the normal (non-adversarial) model, we adopt a stratified splitting procedure to divide the data into training, validation, and test sets. This approach is designed to preserve the original class distribution across all splits. Specifically, 50\% of the training data is used to train the defensive models, while 10\% is allocated for validation purposes. The remaining data constitutes the test set.

The test set is then further subdivided based on the research question to ensure balanced representation:

\begin{itemize}
    \item \textbf{RQ1 and RQ2:} The test dataset is split equally into a calibration set and a test set, each comprising 50\% of the original test data.
    \item \textbf{RQ3:} The test dataset is partitioned into three subsets: 25\% for calibration, 25\% for evaluation, and 50\% for the final test set.
\end{itemize}

All experimental settings use an alpha value of 0.1. This stratified splitting procedure ensures that each class is appropriately utilized for training, validation, and evaluation, thereby supporting consistent performance assessment of the defensive models under different adversarial attack scenarios.

The experimental setup comprises several key components, including the model architecture, training parameters, and adversarial example generation. We employ ResNet18 adapted for single-channel input and appropriate output layers as our model architecture. Training parameters include optimization using Stochastic Gradient Descent (SGD) with learning rate scheduling, a batch size set to 128, and training conducted over 60 epochs. Adversarial examples are generated using TorchAttacks, specifically implementing FGSM, PGD, and SPSA attacks during both the training and evaluation phases.

\subsubsection{Adversarial Attack Types}
Adversarial attacks are implemented using TorchAttacks \cite{kim2020torchattacks}. The selected attacks include FGSM, PGD, SPSA, CW and Clean. Each attack type is designed to target the defensive model against which it is trained, facilitating comprehensive evaluation of model robustness across a variety of adversarial strategies. Employing multiple attack methodologies allows for a thorough assessment of a model's vulnerabilities and ensures that defenses are not tailored to only one specific type of adversarial manipulation.

\subsubsection{Conformal Prediction Calibration and Evaluation}
Conformal Prediction is performed using TorchCP \cite{wei2024torchcp}. To assess the performance of conformal prediction methods, we employ three primary metrics: Coverage, Size, and Size-Stratified Coverage Violation (SSCV). 

The \textbf{Coverage} measures the proportion of test instances in \(\Itest\) where the true label is contained within the prediction set \(\Gamma(x_i)\), and is defined as
\begin{equation}\label{eq:evaluate_coverage}
    \textrm{Coverage} = \frac{1}{|\Itest|} \sum_{i \in \Itest} \mathbbm{1}\left(y_i \in \Gamma(x_i)\right).
\end{equation}
A higher coverage indicates that the prediction sets reliably contain the true labels.

The \textbf{Size} metric calculates the average number of labels in the prediction sets across all test instances,
\begin{equation}\label{eq:size}
    \textrm{Size} = \frac{1}{|\Itest|} \sum_{i \in \Itest} |\Gamma(x_i)|,
\end{equation}
where smaller sizes denote more precise and informative predictions.

The \textbf{Size-Stratified Coverage Violation} (SSCV) \cite{angelopoulosuncertainty} evaluates the consistency of coverage across different prediction set sizes. It is defined as
\begin{align}\label{eq:sscv}
    \text{SSCV}(\Gamma, \{S_j\}_{j=1}^s) = \sup_{j \in [s]} \left|\frac{|\{i \in J_j: y_i \in \Gamma(x_i)\}|}{|J_j|} - (1 - \alpha)\right|,
\end{align}
where \(\{S_j\}_{j=1}^s\) partitions the possible prediction set sizes, and \(J_j = \{i \in \Itest : |\Gamma(x_i)| \in S_j\}\). A smaller SSCV indicates more stable coverage across different set sizes.

These metrics balance the trade-off between achieving the desired coverage probability and maintaining informative prediction sets, while ensuring that the coverage guarantees of conformal prediction hold regardless of the underlying model's accuracy.

\subsection{Results}
The trained models are evaluated against both known and unknown adversarial attacks to assess their robustness and the effectiveness of the conformal prediction sets. For RQ1, models are tested against the known attack applied during calibration. For RQ2, models are evaluated against unseen attacks to determine their generalizability. For RQ3, the game-theoretic approach is analyzed to identify optimal defensive strategies under adversarial conditions. For all experiments, we report the average and standard deviation across a number of splits to ensure statistical robustness.
\paragraph{Conformal Prediction Set Efficiency} The model with the best accuracy does not necessarily produce the most efficient conformal prediction sets. This discrepancy may be influenced by the chosen non-conformity score function. In our experiments, we evaluated four conformal prediction (CP) methods: the standard APS \cite{NEURIPS2020_244edd7e}, RSCP \cite{gendler2021adversarially},  VRCP-I \cite{jeary2024verifiably}, and VRCP-C \cite{jeary2024verifiably}. The standard APS method uses the cumulative probability of labels that have the same or lower estimated probability than the target label, emphasizing the tail probabilities. RSCP extends this approach by incorporating Gaussian noise into the inputs and averaging the scores across multiple samples, effectively “smoothing” the estimated scores. In contrast, VRCP-I and VRCP-C introduce a robustness component by applying adversarial perturbations to the inputs within a specified $\epsilon$-ball. Specifically, VRCP-I computes the minimum score over multiple perturbed copies (thus capturing a more conservative score for each class), while VRCP-C focuses on the worst-case score across these perturbations. These methods collectively allow us to dynamically adjust prediction sets to better capture uncertainty, with the RSCP and VRCP variants enhancing robustness under adversarial conditions, which may result in larger prediction sets for accuracy-driven models.

\subsubsection{Results for RQ1 and RQ2}
In our experimental evaluation, we compared two distinct scenarios: one in which the attack’s nature is known (RQ1), and another in which we assume that an unknown attack belongs to one of five predefined categories (FGSM, PGD, SPSA, CW, or Clean) (RQ2). Both experiments are conducted under the same conformal prediction (CP) framework, which enables us to directly assess how the CP coverage and prediction set sizes vary under different conditions. Additionally, in Tables 3 to 6, the bold entries highlight the best defense model for a particular CP method under each type of attack, while underlined entries indicate the corresponding metric for the best CP method under a specific attack and defense model combination.

For RQ1, known attacks are applied to the test data. For each attack \( g_j \) (where \( j = 1, \dots, m \)), we train a corresponding defensive model \( f_j \) alongside a baseline model \( f_0 \). In contrast, for RQ2 we assume that the unknown attack belongs to one of the five designated categories and apply it to the test data. We then employ Algorithm \ref{alg:rq2} to determine the maximum threshold obtained across the calibration sets under different attacks, thereby facilitating the construction of a conservative CP that guarantees the required coverage level. To better highlight the connections and distinctions between RQ1 and RQ2, we discuss the results from both scenarios together.
    
\begin{table}
\centering
\caption{\textbf{RQ1:} Mean and Standard Deviation of Coverage, Size, and SSCV for \textbf{OrganAMNIST}}
\resizebox{0.9\textwidth}{!}{
\begin{tabular}{ll *{4}{c} *{4}{c} *{4}{c}}
\cmidrule(lr){1-14}
\makecell{Attack} & \makecell{Defensive} &
\multicolumn{4}{c}{\makecell{Coverage (\%)}} &
\multicolumn{4}{c}{Size} &
\multicolumn{4}{c}{SSCV} \\
\cmidrule(lr){3-6} \cmidrule(lr){7-10} \cmidrule(lr){11-14}
 & Model & APS & RSCP & VRCP-I & VRCP-C & APS & RSCP & VRCP-I & VRCP-C & APS & RSCP & VRCP-I & VRCP-C \\
\cmidrule(lr){1-14}
\multirow{4}{*}{FGSM} & Normal & 89.88 $\pm$ 0.51 & 89.40 $\pm$ 0.52 & \textbf{90.27 $\pm$ 0.50} & \underline{91.70 $\pm$ 0.46} & \underline{8.08 $\pm$ 0.05} & 8.13 $\pm$ 0.05 & 8.08 $\pm$ 0.05 & 8.39 $\pm$ 0.05 & 0.10 $\pm$ 0.03 & \underline{0.10 $\pm$ 0.02} & 0.10 $\pm$ 0.03 & 0.10 $\pm$ 0.02 \\
 & FGSM & 89.82 $\pm$ 0.51 & 90.19 $\pm$ 0.50 & 89.45 $\pm$ 0.52 & \underline{92.29 $\pm$ 0.45} & 1.48 $\pm$ 0.02 & 1.82 $\pm$ 0.02 & \underline{1.36 $\pm$ 0.01} & \textbf{1.67 $\pm$ 0.02} & 0.06 $\pm$ 0.01 & 0.07 $\pm$ 0.02 & 0.08 $\pm$ 0.02 & \underline{\textbf{0.04 $\pm$ 0.00}} \\
 & PGD & 90.07 $\pm$ 0.50 & 90.44 $\pm$ 0.49 & 90.19 $\pm$ 0.50 & \underline{\textbf{94.21 $\pm$ 0.39}} & \textbf{1.43 $\pm$ 0.02} & \textbf{1.59 $\pm$ 0.02} & \underline{\textbf{1.25 $\pm$ 0.01}} & 1.72 $\pm$ 0.02 & \underline{\textbf{0.05 $\pm$ 0.01}} & 0.08 $\pm$ 0.02 & 0.09 $\pm$ 0.02 & 0.06 $\pm$ 0.01 \\
 & SPSA & \textbf{90.41 $\pm$ 0.49} & \textbf{90.83 $\pm$ 0.48} & 89.90 $\pm$ 0.51 & \underline{92.21 $\pm$ 0.45} & 1.53 $\pm$ 0.02 & 2.16 $\pm$ 0.02 & \underline{1.44 $\pm$ 0.01} & 1.71 $\pm$ 0.02 & 0.05 $\pm$ 0.01 & \textbf{0.06 $\pm$ 0.01} & \textbf{0.08 $\pm$ 0.02} & \underline{0.04 $\pm$ 0.01} \\
\cmidrule(lr){1-14}
\multirow{4}{*}{PGD} & Normal & \textbf{90.47 $\pm$ 0.49} & 90.19 $\pm$ 0.50 & 89.99 $\pm$ 0.50 & \underline{92.27 $\pm$ 0.45} & 8.90 $\pm$ 0.04 & \underline{8.56 $\pm$ 0.05} & 8.89 $\pm$ 0.04 & 9.22 $\pm$ 0.04 & 0.53 $\pm$ 0.11 & \underline{0.22 $\pm$ 0.04} & 0.54 $\pm$ 0.11 & 0.53 $\pm$ 0.12 \\
 & FGSM & 90.02 $\pm$ 0.50 & \textbf{90.30 $\pm$ 0.50} & \textbf{90.52 $\pm$ 0.49} & \underline{93.14 $\pm$ 0.42} & 1.43 $\pm$ 0.02 & 1.80 $\pm$ 0.02 & \underline{1.36 $\pm$ 0.01} & \textbf{1.67 $\pm$ 0.02} & 0.05 $\pm$ 0.01 & 0.08 $\pm$ 0.02 & \textbf{0.06 $\pm$ 0.01} & \underline{\textbf{0.04 $\pm$ 0.01}} \\
 & PGD & 89.74 $\pm$ 0.51 & 90.04 $\pm$ 0.50 & 90.27 $\pm$ 0.50 & \underline{\textbf{93.79 $\pm$ 0.40}} & \textbf{1.43 $\pm$ 0.02} & \textbf{1.54 $\pm$ 0.02} & \underline{\textbf{1.25 $\pm$ 0.01}} & 1.70 $\pm$ 0.02 & \underline{\textbf{0.04 $\pm$ 0.01}} & 0.09 $\pm$ 0.02 & 0.09 $\pm$ 0.03 & 0.05 $\pm$ 0.01 \\
 & SPSA & 90.35 $\pm$ 0.50 & 90.13 $\pm$ 0.50 & 90.47 $\pm$ 0.49 & \underline{92.77 $\pm$ 0.43} & 1.50 $\pm$ 0.02 & 2.05 $\pm$ 0.02 & \underline{1.41 $\pm$ 0.01} & 1.70 $\pm$ 0.02 & 0.06 $\pm$ 0.01 & \textbf{0.05 $\pm$ 0.01} & 0.07 $\pm$ 0.02 & \underline{0.05 $\pm$ 0.01} \\
\cmidrule(lr){1-14}
\multirow{4}{*}{SPSA} & Normal & 89.15 $\pm$ 0.52 & 89.45 $\pm$ 0.52 & 89.31 $\pm$ 0.52 & \underline{91.11 $\pm$ 0.48} & \underline{7.69 $\pm$ 0.05} & 7.87 $\pm$ 0.05 & 7.70 $\pm$ 0.05 & 7.98 $\pm$ 0.05 & \underline{0.10 $\pm$ 0.02} & 0.13 $\pm$ 0.03 & 0.10 $\pm$ 0.02 & 0.10 $\pm$ 0.02 \\
 & FGSM & \textbf{90.38 $\pm$ 0.49} & 90.21 $\pm$ 0.50 & 89.99 $\pm$ 0.50 & \underline{93.73 $\pm$ 0.41} & 1.40 $\pm$ 0.02 & 1.66 $\pm$ 0.02 & \underline{1.29 $\pm$ 0.01} & 1.61 $\pm$ 0.02 & \underline{0.04 $\pm$ 0.01} & 0.07 $\pm$ 0.01 & 0.10 $\pm$ 0.02 & 0.05 $\pm$ 0.01 \\
 & PGD & 89.85 $\pm$ 0.51 & \textbf{90.27 $\pm$ 0.50} & 90.10 $\pm$ 0.50 & \underline{\textbf{93.90 $\pm$ 0.40}} & \textbf{1.37 $\pm$ 0.01} & \textbf{1.47 $\pm$ 0.02} & \underline{\textbf{1.19 $\pm$ 0.01}} & 1.64 $\pm$ 0.02 & \underline{\textbf{0.01 $\pm$ 0.00}} & 0.09 $\pm$ 0.02 & 0.10 $\pm$ 0.03 & 0.05 $\pm$ 0.01 \\
 & SPSA & 90.35 $\pm$ 0.50 & 89.93 $\pm$ 0.50 & \textbf{91.20 $\pm$ 0.48} & \underline{93.05 $\pm$ 0.43} & 1.41 $\pm$ 0.02 & 1.94 $\pm$ 0.02 & \underline{1.34 $\pm$ 0.01} & \textbf{1.60 $\pm$ 0.02} & 0.07 $\pm$ 0.02 & \textbf{0.07 $\pm$ 0.01} & \textbf{0.06 $\pm$ 0.01} & \underline{\textbf{0.05 $\pm$ 0.01}} \\
\cmidrule(lr){1-14}
\multirow{4}{*}{CW} & Normal & 89.74 $\pm$ 0.51 & 89.40 $\pm$ 0.52 & 89.65 $\pm$ 0.51 & \underline{91.14 $\pm$ 0.48} & \underline{7.74 $\pm$ 0.05} & 7.89 $\pm$ 0.05 & 7.74 $\pm$ 0.05 & 7.99 $\pm$ 0.05 & 0.19 $\pm$ 0.04 & \underline{0.13 $\pm$ 0.02} & 0.19 $\pm$ 0.04 & 0.14 $\pm$ 0.03 \\
 & FGSM & 89.99 $\pm$ 0.50 & 90.52 $\pm$ 0.49 & 90.33 $\pm$ 0.50 & \underline{93.00 $\pm$ 0.43} & \textbf{1.35 $\pm$ 0.01} & 1.69 $\pm$ 0.02 & \underline{1.26 $\pm$ 0.01} & \textbf{1.57 $\pm$ 0.02} & 0.05 $\pm$ 0.01 & 0.06 $\pm$ 0.01 & \textbf{0.07 $\pm$ 0.02} & \underline{\textbf{0.04 $\pm$ 0.01}} \\
 & PGD & 89.90 $\pm$ 0.51 & 90.80 $\pm$ 0.48 & 89.93 $\pm$ 0.50 & \underline{\textbf{94.52 $\pm$ 0.38}} & 1.38 $\pm$ 0.01 & \textbf{1.48 $\pm$ 0.02} & \underline{\textbf{1.18 $\pm$ 0.01}} & 1.66 $\pm$ 0.02 & \underline{\textbf{0.05 $\pm$ 0.01}} & 0.10 $\pm$ 0.02 & 0.13 $\pm$ 0.03 & 0.06 $\pm$ 0.02 \\
 & SPSA & \textbf{90.55 $\pm$ 0.49} & \textbf{90.89 $\pm$ 0.48} & \textbf{90.80 $\pm$ 0.48} & \underline{93.62 $\pm$ 0.41} & 1.40 $\pm$ 0.02 & 1.98 $\pm$ 0.02 & \underline{1.32 $\pm$ 0.01} & 1.61 $\pm$ 0.02 & 0.06 $\pm$ 0.02 & \textbf{0.05 $\pm$ 0.01} & 0.08 $\pm$ 0.01 & \underline{0.05 $\pm$ 0.01} \\
\cmidrule(lr){1-14}
\multirow{4}{*}{Clean} & Normal & 90.24 $\pm$ 0.50 & \textbf{91.84 $\pm$ 0.46} & \textbf{91.03 $\pm$ 0.48} & \underline{93.17 $\pm$ 0.42} & 1.27 $\pm$ 0.01 & 1.26 $\pm$ 0.01 & \underline{1.19 $\pm$ 0.01} & 1.44 $\pm$ 0.02 & 0.10 $\pm$ 0.02 & \textbf{0.07 $\pm$ 0.01} & 0.11 $\pm$ 0.02 & \underline{\textbf{0.05 $\pm$ 0.01}} \\
 & FGSM & 90.72 $\pm$ 0.49 & 89.85 $\pm$ 0.51 & 89.82 $\pm$ 0.51 & \underline{92.86 $\pm$ 0.43} & 1.29 $\pm$ 0.01 & 1.19 $\pm$ 0.01 & \underline{\textbf{1.17 $\pm$ 0.01}} & 1.42 $\pm$ 0.02 & 0.09 $\pm$ 0.03 & 0.15 $\pm$ 0.04 & 0.15 $\pm$ 0.04 & \underline{0.06 $\pm$ 0.01} \\
 & PGD & \textbf{90.97 $\pm$ 0.48} & 90.41 $\pm$ 0.49 & 90.86 $\pm$ 0.48 & \underline{\textbf{94.09 $\pm$ 0.40}} & \textbf{1.24 $\pm$ 0.01} & \underline{\textbf{1.16 $\pm$ 0.01}} & 1.17 $\pm$ 0.01 & 1.42 $\pm$ 0.02 & \underline{\textbf{0.04 $\pm$ 0.01}} & 0.09 $\pm$ 0.03 & \textbf{0.08 $\pm$ 0.02} & 0.05 $\pm$ 0.01 \\
 & SPSA & 90.49 $\pm$ 0.49 & 89.85 $\pm$ 0.51 & 90.49 $\pm$ 0.49 & \underline{93.81 $\pm$ 0.40} & 1.26 $\pm$ 0.01 & 1.19 $\pm$ 0.01 & \underline{1.18 $\pm$ 0.01} & \textbf{1.42 $\pm$ 0.02} & 0.07 $\pm$ 0.02 & 0.11 $\pm$ 0.03 & 0.13 $\pm$ 0.03 & \underline{0.05 $\pm$ 0.01} \\
\cmidrule(lr){1-14}
\label{RQ1OrganAMNIST}
\end{tabular}
}
\end{table}
\begin{table}
\centering
\caption{\textbf{RQ2:} Mean and Standard Deviation of Coverage, Size, and SSCV for \textbf{OrganAMNIST}}
\resizebox{0.9\textwidth}{!}{
\begin{tabular}{ll *{4}{c} *{4}{c} *{4}{c}}
\cmidrule(lr){1-14}
\makecell{Test} & \makecell{Defensive} &
\multicolumn{4}{c}{\makecell{Coverage (\%)}} &
\multicolumn{4}{c}{Size} &
\multicolumn{4}{c}{SSCV} \\
\cmidrule(lr){3-6} \cmidrule(lr){7-10} \cmidrule(lr){11-14}
Attack & Model & APS & RSCP & VRCP-I & VRCP-C & APS & RSCP & VRCP-I & VRCP-C & APS & RSCP & VRCP-I & VRCP-C \\
\cmidrule(lr){1-14}
\multirow{4}{*}{FGSM} & Normal & \textbf{97.13 $\pm$ 0.28} & \textbf{95.39 $\pm$ 0.35} & \textbf{97.33 $\pm$ 0.27} & \underline{\textbf{98.23 $\pm$ 0.22}} & 9.58 $\pm$ 0.04 & \underline{9.05 $\pm$ 0.04} & 9.46 $\pm$ 0.04 & 9.82 $\pm$ 0.04 & \underline{0.10 $\pm$ 0.02} & 0.10 $\pm$ 0.02 & 0.10 $\pm$ 0.01 & 0.10 $\pm$ 0.02 \\
 & FGSM & 90.07 $\pm$ 0.50 & 90.02 $\pm$ 0.50 & 86.05 $\pm$ 0.58 & \underline{93.48 $\pm$ 0.41} & \textbf{1.46 $\pm$ 0.02} & \textbf{1.45 $\pm$ 0.02} & \underline{1.28 $\pm$ 0.01} & \textbf{1.67 $\pm$ 0.02} & \underline{\textbf{0.03 $\pm$ 0.01}} & 0.04 $\pm$ 0.01 & 0.09 $\pm$ 0.02 & \textbf{0.05 $\pm$ 0.01} \\
 & PGD & 92.18 $\pm$ 0.45 & 91.34 $\pm$ 0.47 & 83.89 $\pm$ 0.62 & \underline{94.66 $\pm$ 0.38} & 1.55 $\pm$ 0.02 & 1.48 $\pm$ 0.02 & \underline{\textbf{1.22 $\pm$ 0.01}} & 1.79 $\pm$ 0.02 & 0.05 $\pm$ 0.01 & \underline{\textbf{0.03 $\pm$ 0.00}} & \textbf{0.08 $\pm$ 0.01} & 0.06 $\pm$ 0.02 \\
 & SPSA & 90.19 $\pm$ 0.50 & 92.01 $\pm$ 0.45 & 87.82 $\pm$ 0.55 & \underline{93.28 $\pm$ 0.42} & 1.50 $\pm$ 0.02 & 1.65 $\pm$ 0.02 & \underline{1.37 $\pm$ 0.01} & 1.76 $\pm$ 0.02 & 0.07 $\pm$ 0.02 & \underline{0.04 $\pm$ 0.00} & 0.08 $\pm$ 0.02 & 0.05 $\pm$ 0.01 \\
\cmidrule(lr){1-14}
\multirow{4}{*}{PGD} & Normal & 90.04 $\pm$ 0.50 & 83.10 $\pm$ 0.63 & \textbf{88.33 $\pm$ 0.54} & \underline{92.07 $\pm$ 0.45} & 8.93 $\pm$ 0.04 & \underline{8.16 $\pm$ 0.04} & 8.76 $\pm$ 0.04 & 9.29 $\pm$ 0.04 & 0.54 $\pm$ 0.12 & 0.53 $\pm$ 0.12 & \underline{0.47 $\pm$ 0.10} & 0.75 $\pm$ 0.17 \\
 & FGSM & 90.64 $\pm$ 0.49 & 90.72 $\pm$ 0.49 & 86.42 $\pm$ 0.57 & \underline{92.58 $\pm$ 0.44} & \textbf{1.47 $\pm$ 0.02} & \textbf{1.47 $\pm$ 0.02} & \underline{1.28 $\pm$ 0.01} & \textbf{1.66 $\pm$ 0.02} & \underline{\textbf{0.03 $\pm$ 0.00}} & 0.05 $\pm$ 0.01 & \textbf{0.08 $\pm$ 0.02} & \textbf{0.04 $\pm$ 0.01} \\
 & PGD & \textbf{92.60 $\pm$ 0.44} & 90.75 $\pm$ 0.49 & 83.61 $\pm$ 0.62 & \underline{\textbf{95.19 $\pm$ 0.36}} & 1.57 $\pm$ 0.02 & 1.48 $\pm$ 0.02 & \underline{\textbf{1.21 $\pm$ 0.01}} & 1.80 $\pm$ 0.02 & 0.04 $\pm$ 0.01 & \underline{\textbf{0.03 $\pm$ 0.00}} & 0.09 $\pm$ 0.02 & 0.06 $\pm$ 0.01 \\
 & SPSA & 90.61 $\pm$ 0.49 & \textbf{92.29 $\pm$ 0.45} & 87.74 $\pm$ 0.55 & \underline{93.14 $\pm$ 0.42} & 1.51 $\pm$ 0.02 & 1.65 $\pm$ 0.02 & \underline{1.36 $\pm$ 0.01} & 1.73 $\pm$ 0.02 & 0.05 $\pm$ 0.01 & \underline{0.04 $\pm$ 0.00} & 0.09 $\pm$ 0.03 & 0.04 $\pm$ 0.01 \\
\cmidrule(lr){1-14}
\multirow{4}{*}{SPSA} & Normal & \textbf{97.72 $\pm$ 0.25} & \textbf{96.26 $\pm$ 0.32} & \textbf{97.53 $\pm$ 0.26} & \underline{\textbf{98.48 $\pm$ 0.21}} & 9.46 $\pm$ 0.04 & \underline{8.96 $\pm$ 0.05} & 9.30 $\pm$ 0.05 & 9.62 $\pm$ 0.04 & \underline{0.10 $\pm$ 0.01} & 0.10 $\pm$ 0.02 & 0.10 $\pm$ 0.02 & 0.10 $\pm$ 0.01 \\
 & FGSM & 90.58 $\pm$ 0.49 & 91.11 $\pm$ 0.48 & 87.04 $\pm$ 0.56 & \underline{93.42 $\pm$ 0.42} & \textbf{1.43 $\pm$ 0.02} & \textbf{1.43 $\pm$ 0.02} & \underline{1.27 $\pm$ 0.01} & \textbf{1.62 $\pm$ 0.02} & 0.05 $\pm$ 0.01 & \underline{\textbf{0.03 $\pm$ 0.01}} & \textbf{0.08 $\pm$ 0.02} & \textbf{0.05 $\pm$ 0.01} \\
 & PGD & 92.63 $\pm$ 0.44 & 91.37 $\pm$ 0.47 & 83.97 $\pm$ 0.62 & \underline{95.22 $\pm$ 0.36} & 1.54 $\pm$ 0.02 & 1.45 $\pm$ 0.02 & \underline{\textbf{1.19 $\pm$ 0.01}} & 1.77 $\pm$ 0.02 & \underline{0.04 $\pm$ 0.01} & 0.06 $\pm$ 0.01 & 0.09 $\pm$ 0.01 & 0.06 $\pm$ 0.02 \\
 & SPSA & 91.84 $\pm$ 0.46 & 92.32 $\pm$ 0.45 & 88.67 $\pm$ 0.53 & \underline{93.93 $\pm$ 0.40} & 1.50 $\pm$ 0.02 & 1.60 $\pm$ 0.02 & \underline{1.33 $\pm$ 0.01} & 1.70 $\pm$ 0.02 & \underline{\textbf{0.03 $\pm$ 0.01}} & 0.04 $\pm$ 0.01 & 0.08 $\pm$ 0.02 & 0.06 $\pm$ 0.01 \\
\cmidrule(lr){1-14}
\multirow{4}{*}{CW} & Normal & \textbf{97.64 $\pm$ 0.25} & \textbf{96.54 $\pm$ 0.31} & \textbf{97.22 $\pm$ 0.28} & \underline{\textbf{98.17 $\pm$ 0.22}} & 9.34 $\pm$ 0.05 & \underline{8.87 $\pm$ 0.05} & 9.20 $\pm$ 0.05 & 9.50 $\pm$ 0.05 & \underline{0.10 $\pm$ 0.01} & 0.10 $\pm$ 0.02 & 0.10 $\pm$ 0.02 & 0.10 $\pm$ 0.01 \\
 & FGSM & 91.48 $\pm$ 0.47 & 90.83 $\pm$ 0.48 & 87.60 $\pm$ 0.55 & \underline{93.19 $\pm$ 0.42} & \textbf{1.43 $\pm$ 0.02} & \textbf{1.41 $\pm$ 0.02} & \underline{1.24 $\pm$ 0.01} & \textbf{1.62 $\pm$ 0.02} & \underline{0.04 $\pm$ 0.00} & 0.05 $\pm$ 0.01 & 0.08 $\pm$ 0.02 & \textbf{0.05 $\pm$ 0.01} \\
 & PGD & 91.99 $\pm$ 0.46 & 91.51 $\pm$ 0.47 & 83.94 $\pm$ 0.62 & \underline{94.69 $\pm$ 0.38} & 1.52 $\pm$ 0.02 & 1.46 $\pm$ 0.02 & \underline{\textbf{1.20 $\pm$ 0.01}} & 1.75 $\pm$ 0.02 & \underline{\textbf{0.03 $\pm$ 0.00}} & \textbf{0.04 $\pm$ 0.00} & \textbf{0.07 $\pm$ 0.02} & 0.06 $\pm$ 0.01 \\
 & SPSA & 91.20 $\pm$ 0.48 & 93.14 $\pm$ 0.42 & 89.43 $\pm$ 0.52 & \underline{93.87 $\pm$ 0.40} & 1.47 $\pm$ 0.02 & 1.59 $\pm$ 0.02 & \underline{1.31 $\pm$ 0.01} & 1.67 $\pm$ 0.02 & 0.05 $\pm$ 0.01 & \underline{0.05 $\pm$ 0.00} & 0.08 $\pm$ 0.02 & 0.06 $\pm$ 0.01 \\
\cmidrule(lr){1-14}
\multirow{4}{*}{Clean} & Normal & \textbf{99.69 $\pm$ 0.09} & \textbf{99.30 $\pm$ 0.14} & \underline{\textbf{99.80 $\pm$ 0.07}} & \textbf{99.66 $\pm$ 0.10} & 5.01 $\pm$ 0.07 & \underline{4.38 $\pm$ 0.06} & 4.80 $\pm$ 0.07 & 5.26 $\pm$ 0.07 & \underline{0.10 $\pm$ 0.00} & 0.10 $\pm$ 0.00 & \textbf{0.10 $\pm$ 0.00} & 0.10 $\pm$ 0.00 \\
 & FGSM & 91.20 $\pm$ 0.48 & 91.11 $\pm$ 0.48 & 87.35 $\pm$ 0.56 & \underline{93.22 $\pm$ 0.42} & 1.29 $\pm$ 0.01 & 1.29 $\pm$ 0.01 & \underline{1.18 $\pm$ 0.01} & 1.45 $\pm$ 0.02 & 0.10 $\pm$ 0.02 & \textbf{0.05 $\pm$ 0.01} & 0.13 $\pm$ 0.04 & \underline{0.05 $\pm$ 0.01} \\
 & PGD & 90.66 $\pm$ 0.49 & 88.70 $\pm$ 0.53 & 79.53 $\pm$ 0.68 & \underline{93.05 $\pm$ 0.43} & \textbf{1.24 $\pm$ 0.01} & \textbf{1.18 $\pm$ 0.01} & \underline{\textbf{1.01 $\pm$ 0.01}} & \textbf{1.37 $\pm$ 0.01} & 0.06 $\pm$ 0.02 & 0.07 $\pm$ 0.02 & 0.15 $\pm$ 0.01 & \underline{\textbf{0.05 $\pm$ 0.01}} \\
 & SPSA & 92.18 $\pm$ 0.45 & 94.09 $\pm$ 0.40 & 89.74 $\pm$ 0.51 & \underline{94.35 $\pm$ 0.39} & 1.34 $\pm$ 0.01 & 1.45 $\pm$ 0.02 & \underline{1.23 $\pm$ 0.01} & 1.51 $\pm$ 0.02 & \underline{\textbf{0.05 $\pm$ 0.01}} & 0.06 $\pm$ 0.01 & 0.10 $\pm$ 0.03 & 0.06 $\pm$ 0.01 \\
\cmidrule(lr){1-14}
\label{RQ2OrganAMNIST}
\end{tabular}
}
\end{table}

First, on the OrganAMNIST dataset as shown in Tables \ref{RQ1OrganAMNIST} and \ref{RQ2OrganAMNIST}, for RQ1 the CP method achieves effective coverage under every defensive model for each attack. In particular, VRCP-I produces the smallest prediction set sizes, VRCP-C achieves the highest coverage, and APS exhibits superior SSCV performance. Moreover, it is noteworthy that although it is intuitive to expect that a model adversarially trained against a specific attack would perform best against that attack, various factors—such as the attack's complexity, model capacity, and the diversity of adversarial perturbations—can lead some defensive models to outperform others, even without being explicitly trained for that specific attack.

For RQ2, most CP methods require a larger \(q\) value compared to the RQ1 condition, leading to CP coverages as high as 99\%. In comparison to the other CP methods, APS maintains a more stable coverage of around 90\% while also delivering better SSCV. RSCP, on the other hand, is more conservative, achieving higher coverage at the expense of larger prediction set sizes. The VRCP-C method outperforms RSCP in terms of coverage while generating smaller prediction sets. Although VRCP-I results in the smallest set sizes, its coverage is less stable due to the limited number of perturbed samples generated by its self-imposed perturbation approach, which may not align consistently with the prediction set samples.
\begin{table}
\centering
\caption{\textbf{RQ1:} Mean and Standard Deviation of Coverage, Size, and SSCV for \textbf{PathMNIST}}
\resizebox{0.9\textwidth}{!}{
\begin{tabular}{ll *{4}{c} *{4}{c} *{4}{c}}
\cmidrule(lr){1-14}
\makecell{Attack} & \makecell{Defensive} &
\multicolumn{4}{c}{\makecell{Coverage (\%)}} &
\multicolumn{4}{c}{Size} &
\multicolumn{4}{c}{SSCV} \\
\cmidrule(lr){3-6} \cmidrule(lr){7-10} \cmidrule(lr){11-14}
 & Model & APS & RSCP & VRCP-I & VRCP-C & APS & RSCP & VRCP-I & VRCP-C & APS & RSCP & VRCP-I & VRCP-C \\
\cmidrule(lr){1-14}
\multirow{4}{*}{FGSM} & Normal & 89.76 $\pm$ 0.80 & 88.86 $\pm$ 0.83 & 89.76 $\pm$ 0.80 & \underline{91.36 $\pm$ 0.74} & 5.81 $\pm$ 0.07 & \textbf{6.26 $\pm$ 0.06} & \underline{5.78 $\pm$ 0.07} & 6.09 $\pm$ 0.07 & 0.58 $\pm$ 0.18 & 0.90 $\pm$ 0.26 & \underline{0.58 $\pm$ 0.18} & 0.61 $\pm$ 0.19 \\
 & FGSM & \textbf{90.81 $\pm$ 0.76} & 90.18 $\pm$ 0.79 & \textbf{90.81 $\pm$ 0.76} & \underline{\textbf{93.94 $\pm$ 0.63}} & 2.99 $\pm$ 0.04 & 6.93 $\pm$ 0.05 & \underline{2.90 $\pm$ 0.04} & 3.38 $\pm$ 0.05 & \underline{0.06 $\pm$ 0.01} & 0.15 $\pm$ 0.07 & 0.07 $\pm$ 0.02 & 0.07 $\pm$ 0.02 \\
 & PGD & 89.21 $\pm$ 0.82 & 89.90 $\pm$ 0.80 & 89.90 $\pm$ 0.80 & \underline{93.38 $\pm$ 0.66} & 3.28 $\pm$ 0.04 & 7.56 $\pm$ 0.03 & \underline{3.17 $\pm$ 0.04} & 3.79 $\pm$ 0.05 & \textbf{0.04 $\pm$ 0.01} & \underline{\textbf{0.00 $\pm$ 0.00}} & 0.10 $\pm$ 0.03 & 0.07 $\pm$ 0.01 \\
 & SPSA & 90.25 $\pm$ 0.78 & \textbf{91.02 $\pm$ 0.75} & 89.62 $\pm$ 0.81 & \underline{92.90 $\pm$ 0.68} & \underline{\textbf{2.87 $\pm$ 0.04}} & 7.88 $\pm$ 0.04 & \textbf{2.87 $\pm$ 0.04} & \textbf{3.22 $\pm$ 0.05} & 0.05 $\pm$ 0.01 & 0.11 $\pm$ 0.05 & \underline{\textbf{0.03 $\pm$ 0.00}} & \textbf{0.06 $\pm$ 0.01} \\
\cmidrule(lr){1-14}
\multirow{4}{*}{PGD} & Normal & 90.95 $\pm$ 0.76 & 89.35 $\pm$ 0.81 & 91.02 $\pm$ 0.75 & \underline{93.80 $\pm$ 0.64} & 6.86 $\pm$ 0.08 & \underline{\textbf{6.28 $\pm$ 0.06}} & 6.94 $\pm$ 0.08 & 7.24 $\pm$ 0.07 & 0.70 $\pm$ 0.21 & \underline{0.31 $\pm$ 0.13} & 0.73 $\pm$ 0.22 & 0.66 $\pm$ 0.19 \\
 & FGSM & \textbf{91.02 $\pm$ 0.75} & \textbf{90.53 $\pm$ 0.77} & \textbf{91.64 $\pm$ 0.73} & \underline{93.73 $\pm$ 0.64} & 3.04 $\pm$ 0.04 & 6.95 $\pm$ 0.05 & \underline{3.00 $\pm$ 0.04} & 3.35 $\pm$ 0.05 & \underline{\textbf{0.05 $\pm$ 0.01}} & 0.20 $\pm$ 0.10 & 0.05 $\pm$ 0.01 & 0.07 $\pm$ 0.02 \\
 & PGD & 90.11 $\pm$ 0.79 & 89.62 $\pm$ 0.81 & 89.76 $\pm$ 0.80 & \underline{\textbf{93.94 $\pm$ 0.63}} & 3.23 $\pm$ 0.04 & 7.58 $\pm$ 0.03 & \underline{3.01 $\pm$ 0.04} & 3.76 $\pm$ 0.05 & 0.05 $\pm$ 0.01 & \underline{\textbf{0.00 $\pm$ 0.00}} & 0.09 $\pm$ 0.02 & 0.07 $\pm$ 0.01 \\
 & SPSA & 89.42 $\pm$ 0.81 & 90.53 $\pm$ 0.77 & 89.83 $\pm$ 0.80 & \underline{92.62 $\pm$ 0.69} & \underline{\textbf{2.73 $\pm$ 0.04}} & 7.90 $\pm$ 0.04 & \textbf{2.80 $\pm$ 0.04} & \textbf{3.01 $\pm$ 0.04} & 0.06 $\pm$ 0.02 & 0.29 $\pm$ 0.14 & \underline{\textbf{0.04 $\pm$ 0.01}} & \textbf{0.06 $\pm$ 0.01} \\
\cmidrule(lr){1-14}
\multirow{4}{*}{SPSA} & Normal & \textbf{91.16 $\pm$ 0.75} & 88.72 $\pm$ 0.84 & \textbf{91.36 $\pm$ 0.74} & \underline{92.27 $\pm$ 0.70} & \underline{4.86 $\pm$ 0.06} & \textbf{5.95 $\pm$ 0.06} & 4.88 $\pm$ 0.06 & 5.27 $\pm$ 0.06 & 0.38 $\pm$ 0.11 & 0.65 $\pm$ 0.18 & \underline{0.37 $\pm$ 0.11} & 0.39 $\pm$ 0.11 \\
 & FGSM & 90.25 $\pm$ 0.78 & 90.81 $\pm$ 0.76 & 90.39 $\pm$ 0.78 & \underline{93.66 $\pm$ 0.64} & 2.99 $\pm$ 0.04 & 6.94 $\pm$ 0.05 & \underline{2.92 $\pm$ 0.04} & 3.35 $\pm$ 0.05 & \underline{0.06 $\pm$ 0.01} & 0.07 $\pm$ 0.03 & 0.07 $\pm$ 0.02 & 0.08 $\pm$ 0.02 \\
 & PGD & 89.42 $\pm$ 0.81 & 90.04 $\pm$ 0.79 & 89.14 $\pm$ 0.82 & \underline{\textbf{93.73 $\pm$ 0.64}} & 3.05 $\pm$ 0.04 & 7.57 $\pm$ 0.03 & \underline{2.80 $\pm$ 0.04} & 3.58 $\pm$ 0.05 & \textbf{0.03 $\pm$ 0.00} & \underline{\textbf{0.00 $\pm$ 0.00}} & 0.09 $\pm$ 0.03 & 0.06 $\pm$ 0.01 \\
 & SPSA & 90.04 $\pm$ 0.79 & \textbf{92.55 $\pm$ 0.69} & 90.18 $\pm$ 0.79 & \underline{93.38 $\pm$ 0.66} & \textbf{2.26 $\pm$ 0.04} & 7.92 $\pm$ 0.04 & \underline{\textbf{2.13 $\pm$ 0.03}} & \textbf{2.58 $\pm$ 0.04} & \underline{0.05 $\pm$ 0.01} & 0.05 $\pm$ 0.01 & \textbf{0.05 $\pm$ 0.01} & \textbf{0.05 $\pm$ 0.01} \\
\cmidrule(lr){1-14}
\multirow{4}{*}{CW} & Normal & 90.39 $\pm$ 0.78 & 89.76 $\pm$ 0.80 & 90.60 $\pm$ 0.77 & \underline{91.57 $\pm$ 0.73} & \underline{5.19 $\pm$ 0.06} & 6.30 $\pm$ 0.05 & 5.27 $\pm$ 0.06 & 5.59 $\pm$ 0.06 & 0.45 $\pm$ 0.13 & 0.46 $\pm$ 0.20 & 0.44 $\pm$ 0.13 & \underline{0.37 $\pm$ 0.11} \\
 & FGSM & \textbf{91.09 $\pm$ 0.75} & 89.00 $\pm$ 0.83 & \textbf{90.67 $\pm$ 0.77} & \underline{93.45 $\pm$ 0.65} & 3.06 $\pm$ 0.04 & 6.03 $\pm$ 0.05 & \underline{2.96 $\pm$ 0.04} & 3.46 $\pm$ 0.05 & \underline{0.05 $\pm$ 0.01} & 0.38 $\pm$ 0.10 & 0.07 $\pm$ 0.02 & 0.07 $\pm$ 0.01 \\
 & PGD & 88.93 $\pm$ 0.83 & 89.55 $\pm$ 0.81 & 89.21 $\pm$ 0.82 & \underline{93.87 $\pm$ 0.63} & 3.00 $\pm$ 0.04 & \textbf{5.24 $\pm$ 0.04} & \underline{2.76 $\pm$ 0.04} & 3.59 $\pm$ 0.05 & \underline{0.03 $\pm$ 0.00} & 0.34 $\pm$ 0.08 & 0.08 $\pm$ 0.02 & 0.06 $\pm$ 0.01 \\
 & SPSA & 89.55 $\pm$ 0.81 & \textbf{89.97 $\pm$ 0.79} & 90.11 $\pm$ 0.79 & \underline{\textbf{94.22 $\pm$ 0.62}} & \textbf{2.18 $\pm$ 0.03} & 6.32 $\pm$ 0.06 & \underline{\textbf{2.09 $\pm$ 0.03}} & \textbf{2.65 $\pm$ 0.04} & \underline{\textbf{0.02 $\pm$ 0.00}} & \textbf{0.32 $\pm$ 0.07} & \textbf{0.03 $\pm$ 0.01} & \textbf{0.06 $\pm$ 0.01} \\
\cmidrule(lr){1-14}
\multirow{4}{*}{Clean} & Normal & 89.90 $\pm$ 0.80 & \textbf{88.65 $\pm$ 0.84} & \textbf{89.97 $\pm$ 0.79} & \underline{92.20 $\pm$ 0.71} & 2.19 $\pm$ 0.03 & 5.36 $\pm$ 0.05 & \underline{2.02 $\pm$ 0.03} & 2.51 $\pm$ 0.04 & 0.05 $\pm$ 0.01 & 0.26 $\pm$ 0.07 & \underline{\textbf{0.04 $\pm$ 0.01}} & 0.05 $\pm$ 0.01 \\
 & FGSM & 89.28 $\pm$ 0.82 & 88.09 $\pm$ 0.85 & 89.00 $\pm$ 0.83 & \underline{93.25 $\pm$ 0.66} & 2.17 $\pm$ 0.04 & 5.25 $\pm$ 0.05 & \underline{1.99 $\pm$ 0.03} & 2.53 $\pm$ 0.04 & \underline{\textbf{0.03 $\pm$ 0.01}} & 0.09 $\pm$ 0.03 & 0.08 $\pm$ 0.02 & \textbf{0.05 $\pm$ 0.01} \\
 & PGD & 89.69 $\pm$ 0.80 & 88.37 $\pm$ 0.85 & 89.21 $\pm$ 0.82 & \underline{\textbf{94.01 $\pm$ 0.63}} & 2.25 $\pm$ 0.04 & \textbf{4.84 $\pm$ 0.05} & \underline{2.11 $\pm$ 0.03} & 2.72 $\pm$ 0.04 & \underline{0.05 $\pm$ 0.01} & 0.29 $\pm$ 0.08 & 0.07 $\pm$ 0.01 & 0.06 $\pm$ 0.02 \\
 & SPSA & \textbf{90.39 $\pm$ 0.78} & 88.16 $\pm$ 0.85 & 89.07 $\pm$ 0.82 & \underline{93.25 $\pm$ 0.66} & \textbf{2.12 $\pm$ 0.03} & 5.20 $\pm$ 0.05 & \underline{\textbf{1.89 $\pm$ 0.03}} & \textbf{2.36 $\pm$ 0.04} & 0.06 $\pm$ 0.01 & \textbf{0.07 $\pm$ 0.02} & \underline{0.05 $\pm$ 0.01} & 0.06 $\pm$ 0.02 \\
\cmidrule(lr){1-14}
\end{tabular}
\label{RQ1PathMNIST}
}
\end{table}
\begin{table}
\centering
\caption{\textbf{RQ2:} Mean and Standard Deviation of Coverage, Size, and SSCV for \textbf{PathMNIST}}
\resizebox{0.9\textwidth}{!}{
\begin{tabular}{ll *{4}{c} *{4}{c} *{4}{c}}
\cmidrule(lr){1-14}
\makecell{Test} & \makecell{Defensive} &
\multicolumn{4}{c}{\makecell{Coverage (\%)}} &
\multicolumn{4}{c}{Size} &
\multicolumn{4}{c}{SSCV} \\
\cmidrule(lr){3-6} \cmidrule(lr){7-10} \cmidrule(lr){11-14}
Attack & Model & APS & RSCP & VRCP-I & VRCP-C & APS & RSCP & VRCP-I & VRCP-C & APS & RSCP & VRCP-I & VRCP-C \\
\cmidrule(lr){1-14}
\multirow{4}{*}{FGSM} & Normal & \textbf{98.40 $\pm$ 0.33} & 88.86 $\pm$ 0.83 & \textbf{97.98 $\pm$ 0.37} & \underline{\textbf{99.93 $\pm$ 0.07}} & 7.68 $\pm$ 0.07 & \underline{\textbf{5.71 $\pm$ 0.07}} & 7.59 $\pm$ 0.07 & 8.95 $\pm$ 0.01 & \underline{0.47 $\pm$ 0.13} & 0.58 $\pm$ 0.18 & 0.58 $\pm$ 0.16 & 0.90 $\pm$ 0.40 \\
 & FGSM & 91.23 $\pm$ 0.75 & \underline{99.93 $\pm$ 0.07} & 88.37 $\pm$ 0.85 & 94.36 $\pm$ 0.61 & 3.12 $\pm$ 0.04 & 7.65 $\pm$ 0.03 & \underline{2.75 $\pm$ 0.04} & 3.53 $\pm$ 0.05 & \underline{\textbf{0.05 $\pm$ 0.01}} & \textbf{0.10 $\pm$ 0.00} & 0.08 $\pm$ 0.02 & 0.08 $\pm$ 0.02 \\
 & PGD & 93.87 $\pm$ 0.63 & \underline{\textbf{100.00 $\pm$ 0.00}} & 88.65 $\pm$ 0.84 & 96.59 $\pm$ 0.48 & 3.76 $\pm$ 0.05 & 7.33 $\pm$ 0.06 & \underline{3.17 $\pm$ 0.04} & 4.36 $\pm$ 0.05 & 0.07 $\pm$ 0.01 & 0.10 $\pm$ 0.00 & \underline{\textbf{0.03 $\pm$ 0.01}} & 0.08 $\pm$ 0.01 \\
 & SPSA & 89.14 $\pm$ 0.82 & \underline{100.00 $\pm$ 0.00} & 86.91 $\pm$ 0.89 & 92.27 $\pm$ 0.70 & \textbf{2.90 $\pm$ 0.04} & 8.94 $\pm$ 0.01 & \underline{\textbf{2.64 $\pm$ 0.04}} & \textbf{3.27 $\pm$ 0.05} & 0.06 $\pm$ 0.02 & 0.10 $\pm$ 0.00 & 0.07 $\pm$ 0.02 & \underline{\textbf{0.05 $\pm$ 0.01}} \\
\cmidrule(lr){1-14}
\multirow{4}{*}{PGD} & Normal & 89.55 $\pm$ 0.81 & 56.27 $\pm$ 1.31 & 88.02 $\pm$ 0.86 & \underline{\textbf{98.89 $\pm$ 0.28}} & 6.77 $\pm$ 0.08 & \underline{\textbf{4.08 $\pm$ 0.07}} & 6.56 $\pm$ 0.07 & 8.85 $\pm$ 0.02 & \underline{0.69 $\pm$ 0.21} & 0.73 $\pm$ 0.17 & 0.76 $\pm$ 0.24 & 0.79 $\pm$ 0.35 \\
 & FGSM & 91.09 $\pm$ 0.75 & \underline{\textbf{100.00 $\pm$ 0.00}} & 88.79 $\pm$ 0.83 & 93.04 $\pm$ 0.67 & 3.06 $\pm$ 0.04 & 7.63 $\pm$ 0.03 & \underline{2.74 $\pm$ 0.04} & 3.54 $\pm$ 0.05 & \underline{0.04 $\pm$ 0.00} & \textbf{0.10 $\pm$ 0.00} & 0.09 $\pm$ 0.02 & 0.07 $\pm$ 0.02 \\
 & PGD & \textbf{93.31 $\pm$ 0.66} & \underline{100.00 $\pm$ 0.00} & \textbf{88.93 $\pm$ 0.83} & 96.80 $\pm$ 0.46 & 3.73 $\pm$ 0.05 & 7.31 $\pm$ 0.06 & \underline{3.16 $\pm$ 0.04} & 4.33 $\pm$ 0.05 & 0.07 $\pm$ 0.02 & 0.10 $\pm$ 0.00 & \underline{\textbf{0.04 $\pm$ 0.01}} & 0.08 $\pm$ 0.00 \\
 & SPSA & 89.90 $\pm$ 0.80 & \underline{100.00 $\pm$ 0.00} & 87.40 $\pm$ 0.88 & 92.62 $\pm$ 0.69 & \textbf{2.77 $\pm$ 0.04} & 8.92 $\pm$ 0.01 & \underline{\textbf{2.58 $\pm$ 0.04}} & \textbf{3.14 $\pm$ 0.04} & \underline{\textbf{0.03 $\pm$ 0.01}} & 0.10 $\pm$ 0.00 & 0.09 $\pm$ 0.02 & \textbf{0.05 $\pm$ 0.02} \\
\cmidrule(lr){1-14}
\multirow{4}{*}{SPSA} & Normal & \textbf{99.03 $\pm$ 0.26} & 92.27 $\pm$ 0.70 & \textbf{99.30 $\pm$ 0.22} & \underline{\textbf{99.93 $\pm$ 0.07}} & 7.61 $\pm$ 0.07 & \underline{\textbf{5.52 $\pm$ 0.06}} & 7.55 $\pm$ 0.07 & 8.95 $\pm$ 0.01 & 0.10 $\pm$ 0.00 & 0.43 $\pm$ 0.12 & 0.19 $\pm$ 0.03 & \underline{0.10 $\pm$ 0.00} \\
 & FGSM & 90.67 $\pm$ 0.77 & \underline{99.93 $\pm$ 0.07} & 87.60 $\pm$ 0.87 & 94.50 $\pm$ 0.60 & 3.09 $\pm$ 0.04 & 7.63 $\pm$ 0.03 & \underline{2.70 $\pm$ 0.04} & 3.52 $\pm$ 0.05 & \underline{0.07 $\pm$ 0.01} & \textbf{0.10 $\pm$ 0.00} & 0.08 $\pm$ 0.02 & 0.08 $\pm$ 0.02 \\
 & PGD & 94.57 $\pm$ 0.60 & \underline{\textbf{100.00 $\pm$ 0.00}} & 90.04 $\pm$ 0.79 & 97.14 $\pm$ 0.44 & 3.71 $\pm$ 0.05 & 7.27 $\pm$ 0.06 & \underline{3.14 $\pm$ 0.04} & 4.29 $\pm$ 0.05 & 0.06 $\pm$ 0.01 & 0.10 $\pm$ 0.00 & \underline{\textbf{0.03 $\pm$ 0.00}} & 0.09 $\pm$ 0.01 \\
 & SPSA & 94.64 $\pm$ 0.59 & \underline{100.00 $\pm$ 0.00} & 91.78 $\pm$ 0.72 & 96.31 $\pm$ 0.50 & \textbf{2.70 $\pm$ 0.04} & 8.79 $\pm$ 0.02 & \underline{\textbf{2.43 $\pm$ 0.04}} & \textbf{3.06 $\pm$ 0.05} & \textbf{0.06 $\pm$ 0.02} & 0.10 $\pm$ 0.00 & \underline{0.05 $\pm$ 0.01} & \textbf{0.07 $\pm$ 0.01} \\
\cmidrule(lr){1-14}
\multirow{4}{*}{CW} & Normal & \textbf{99.51 $\pm$ 0.18} & 92.90 $\pm$ 0.68 & \textbf{99.09 $\pm$ 0.25} & \underline{\textbf{99.93 $\pm$ 0.07}} & 8.01 $\pm$ 0.06 & \underline{\textbf{5.95 $\pm$ 0.07}} & 7.91 $\pm$ 0.06 & 8.92 $\pm$ 0.02 & \underline{0.10 $\pm$ 0.00} & 0.36 $\pm$ 0.10 & 0.10 $\pm$ 0.01 & 0.10 $\pm$ 0.00 \\
 & FGSM & 90.95 $\pm$ 0.76 & \underline{\textbf{100.00 $\pm$ 0.00}} & 86.70 $\pm$ 0.90 & 92.90 $\pm$ 0.68 & 3.07 $\pm$ 0.04 & 7.07 $\pm$ 0.06 & \underline{2.69 $\pm$ 0.04} & 3.45 $\pm$ 0.05 & \underline{\textbf{0.04 $\pm$ 0.00}} & \textbf{0.10 $\pm$ 0.00} & 0.09 $\pm$ 0.03 & \textbf{0.07 $\pm$ 0.01} \\
 & PGD & 94.08 $\pm$ 0.62 & \underline{100.00 $\pm$ 0.00} & 90.18 $\pm$ 0.79 & 97.42 $\pm$ 0.42 & 3.65 $\pm$ 0.05 & 7.20 $\pm$ 0.07 & \underline{3.09 $\pm$ 0.04} & 4.27 $\pm$ 0.05 & 0.05 $\pm$ 0.01 & 0.10 $\pm$ 0.00 & \underline{\textbf{0.04 $\pm$ 0.01}} & 0.09 $\pm$ 0.01 \\
 & SPSA & 94.15 $\pm$ 0.62 & \underline{100.00 $\pm$ 0.00} & 93.66 $\pm$ 0.64 & 96.24 $\pm$ 0.50 & \textbf{2.66 $\pm$ 0.04} & 8.79 $\pm$ 0.02 & \underline{\textbf{2.46 $\pm$ 0.04}} & \textbf{3.03 $\pm$ 0.05} & 0.06 $\pm$ 0.01 & 0.10 $\pm$ 0.00 & \underline{0.05 $\pm$ 0.01} & 0.08 $\pm$ 0.01 \\
\cmidrule(lr){1-14}
\multirow{4}{*}{Clean} & Normal & \underline{\textbf{100.00 $\pm$ 0.00}} & 99.86 $\pm$ 0.10 & \textbf{100.00 $\pm$ 0.00} & \textbf{100.00 $\pm$ 0.00} & 8.05 $\pm$ 0.06 & \underline{\textbf{5.64 $\pm$ 0.06}} & 7.92 $\pm$ 0.06 & 8.90 $\pm$ 0.02 & \underline{0.10 $\pm$ 0.00} & \textbf{0.10 $\pm$ 0.00} & 0.10 $\pm$ 0.00 & 0.10 $\pm$ 0.00 \\
 & FGSM & 90.18 $\pm$ 0.79 & \underline{\textbf{100.00 $\pm$ 0.00}} & 88.30 $\pm$ 0.85 & 93.87 $\pm$ 0.63 & 2.31 $\pm$ 0.04 & 6.27 $\pm$ 0.06 & \underline{2.06 $\pm$ 0.03} & \textbf{2.63 $\pm$ 0.04} & \underline{0.04 $\pm$ 0.01} & 0.10 $\pm$ 0.00 & 0.05 $\pm$ 0.01 & 0.06 $\pm$ 0.02 \\
 & PGD & 89.62 $\pm$ 0.81 & \underline{99.93 $\pm$ 0.07} & 82.59 $\pm$ 1.00 & 93.45 $\pm$ 0.65 & \textbf{2.26 $\pm$ 0.04} & 6.02 $\pm$ 0.07 & \underline{\textbf{1.87 $\pm$ 0.03}} & 2.69 $\pm$ 0.04 & \underline{\textbf{0.03 $\pm$ 0.01}} & 0.10 $\pm$ 0.00 & 0.14 $\pm$ 0.03 & \textbf{0.05 $\pm$ 0.01} \\
 & SPSA & 92.20 $\pm$ 0.71 & \underline{100.00 $\pm$ 0.00} & 91.64 $\pm$ 0.73 & 95.61 $\pm$ 0.54 & 2.40 $\pm$ 0.04 & 8.73 $\pm$ 0.02 & \underline{2.24 $\pm$ 0.04} & 2.71 $\pm$ 0.04 & 0.04 $\pm$ 0.01 & 0.10 $\pm$ 0.00 & \underline{\textbf{0.04 $\pm$ 0.01}} & 0.08 $\pm$ 0.02 \\
\cmidrule(lr){1-14}
\label{RQ2PathMNIST}
\end{tabular}
}
\end{table}

Next, we evaluate the performance on the PathMNIST dataset, as presented in Tables \ref{RQ1PathMNIST} and \ref{RQ2PathMNIST}. For RQ1, the CP methods achieve effective coverage for every defensive model under each attack. In most cases, VRCP-I yields the smallest prediction set sizes, VRCP-C achieves the highest coverage, and APS exhibits superior SSCV. For RQ2, most CP methods require a larger \(q\) value compared to RQ1, and notably, the CP coverages under RQ2 even reach 100\%. Compared with the other CP methods, APS maintains a more stable coverage around 90\% with better SSCV. RSCP proves to be more conservative, achieving higher coverage but at the cost of larger prediction set sizes, while VRCP-C offers a balance between the two. Although VRCP-I consistently produces the smallest set sizes, its coverage is unstable for some defensive models.
\begin{table}
\centering
\caption{\textbf{RQ1:} Mean and Standard Deviation of Coverage, Size, and SSCV for \textbf{TissueMNIST}}
\resizebox{0.9\textwidth}{!}{
\begin{tabular}{ll *{4}{c} *{4}{c} *{4}{c}}
\cmidrule(lr){1-14}
\makecell{Attack} & \makecell{Defensive} &
\multicolumn{4}{c}{\makecell{Coverage (\%)}} &
\multicolumn{4}{c}{Size} &
\multicolumn{4}{c}{SSCV} \\
\cmidrule(lr){3-6} \cmidrule(lr){7-10} \cmidrule(lr){11-14}
 & Model & APS & RSCP & VRCP-I & VRCP-C & APS & RSCP & VRCP-I & VRCP-C & APS & RSCP & VRCP-I & VRCP-C \\
\cmidrule(lr){1-14}
\multirow{4}{*}{FGSM} & Normal & 90.08 $\pm$ 0.31 & 89.86 $\pm$ 0.31 & 90.06 $\pm$ 0.31 & \underline{91.78 $\pm$ 0.28} & 5.91 $\pm$ 0.00 & \underline{\textbf{5.67 $\pm$ 0.01}} & 5.87 $\pm$ 0.00 & 6.00 $\pm$ 0.00 & \textbf{0.00 $\pm$ 0.00} & 0.00 $\pm$ 0.00 & \underline{\textbf{0.00 $\pm$ 0.00}} & \textbf{0.02 $\pm$ 0.00} \\
 & FGSM & 90.09 $\pm$ 0.31 & 89.81 $\pm$ 0.31 & 89.84 $\pm$ 0.31 & \underline{\textbf{94.82 $\pm$ 0.23}} & \underline{5.42 $\pm$ 0.01} & 6.73 $\pm$ 0.00 & 6.52 $\pm$ 0.01 & 6.80 $\pm$ 0.01 & 0.20 $\pm$ 0.06 & \underline{0.00 $\pm$ 0.00} & 0.10 $\pm$ 0.03 & 0.05 $\pm$ 0.01 \\
 & PGD & \textbf{90.30 $\pm$ 0.30} & \textbf{90.05 $\pm$ 0.31} & \textbf{90.30 $\pm$ 0.30} & \underline{91.15 $\pm$ 0.29} & \textbf{5.30 $\pm$ 0.01} & 7.58 $\pm$ 0.01 & \underline{5.25 $\pm$ 0.01} & 5.53 $\pm$ 0.01 & 0.10 $\pm$ 0.03 & \underline{\textbf{0.00 $\pm$ 0.00}} & 0.03 $\pm$ 0.01 & 0.04 $\pm$ 0.01 \\
 & SPSA & 89.88 $\pm$ 0.31 & \underline{89.90 $\pm$ 0.31} & 89.53 $\pm$ 0.31 & 72.51 $\pm$ 0.46 & 6.50 $\pm$ 0.01 & 6.13 $\pm$ 0.00 & \underline{\textbf{5.19 $\pm$ 0.01}} & \textbf{5.32 $\pm$ 0.01} & 0.15 $\pm$ 0.05 & \underline{0.00 $\pm$ 0.00} & 0.23 $\pm$ 0.07 & 0.31 $\pm$ 0.07 \\
\cmidrule(lr){1-14}
\multirow{4}{*}{PGD} & Normal & 90.04 $\pm$ 0.31 & 90.08 $\pm$ 0.31 & \textbf{90.19 $\pm$ 0.31} & \underline{91.65 $\pm$ 0.28} & 5.91 $\pm$ 0.00 & \underline{\textbf{5.70 $\pm$ 0.01}} & 5.87 $\pm$ 0.00 & 6.00 $\pm$ 0.00 & \underline{\textbf{0.00 $\pm$ 0.00}} & 0.00 $\pm$ 0.00 & 0.90 $\pm$ 0.45 & \textbf{0.02 $\pm$ 0.00} \\
 & FGSM & 90.09 $\pm$ 0.31 & 89.99 $\pm$ 0.31 & 90.14 $\pm$ 0.31 & \underline{\textbf{94.80 $\pm$ 0.23}} & \underline{5.41 $\pm$ 0.01} & 6.74 $\pm$ 0.00 & 6.54 $\pm$ 0.01 & 6.82 $\pm$ 0.01 & 0.23 $\pm$ 0.07 & \underline{\textbf{0.00 $\pm$ 0.00}} & 0.10 $\pm$ 0.03 & 0.05 $\pm$ 0.02 \\
 & PGD & \textbf{90.32 $\pm$ 0.30} & \textbf{90.46 $\pm$ 0.30} & 90.16 $\pm$ 0.31 & \underline{91.30 $\pm$ 0.29} & \textbf{5.27 $\pm$ 0.01} & 7.60 $\pm$ 0.01 & \underline{5.23 $\pm$ 0.01} & 5.52 $\pm$ 0.01 & 0.06 $\pm$ 0.03 & \underline{0.00 $\pm$ 0.00} & \textbf{0.04 $\pm$ 0.02} & 0.04 $\pm$ 0.01 \\
 & SPSA & 89.83 $\pm$ 0.31 & \underline{89.84 $\pm$ 0.31} & 89.79 $\pm$ 0.31 & 72.35 $\pm$ 0.46 & 6.51 $\pm$ 0.01 & 6.12 $\pm$ 0.00 & \underline{\textbf{5.16 $\pm$ 0.01}} & \textbf{5.31 $\pm$ 0.01} & 0.30 $\pm$ 0.10 & \underline{0.00 $\pm$ 0.00} & 0.20 $\pm$ 0.06 & 0.18 $\pm$ 0.05 \\
\cmidrule(lr){1-14}
\multirow{4}{*}{SPSA} & Normal & 89.59 $\pm$ 0.31 & \textbf{90.04 $\pm$ 0.31} & \textbf{90.06 $\pm$ 0.31} & \underline{91.88 $\pm$ 0.28} & 5.90 $\pm$ 0.00 & \underline{\textbf{5.70 $\pm$ 0.01}} & 5.87 $\pm$ 0.00 & 6.00 $\pm$ 0.00 & \textbf{0.00 $\pm$ 0.00} & \underline{\textbf{0.00 $\pm$ 0.00}} & \textbf{0.00 $\pm$ 0.00} & \textbf{0.02 $\pm$ 0.00} \\
 & FGSM & 90.21 $\pm$ 0.31 & 89.92 $\pm$ 0.31 & 90.05 $\pm$ 0.31 & \underline{\textbf{94.75 $\pm$ 0.23}} & \underline{5.43 $\pm$ 0.01} & 6.73 $\pm$ 0.00 & 6.53 $\pm$ 0.01 & 6.81 $\pm$ 0.01 & 0.19 $\pm$ 0.06 & \underline{0.00 $\pm$ 0.00} & 0.40 $\pm$ 0.13 & 0.05 $\pm$ 0.02 \\
 & PGD & 90.21 $\pm$ 0.31 & 89.46 $\pm$ 0.32 & 90.06 $\pm$ 0.31 & \underline{91.26 $\pm$ 0.29} & \textbf{5.26 $\pm$ 0.01} & 7.57 $\pm$ 0.01 & \underline{5.21 $\pm$ 0.01} & 5.53 $\pm$ 0.01 & 0.07 $\pm$ 0.03 & \underline{0.01 $\pm$ 0.00} & 0.03 $\pm$ 0.01 & 0.06 $\pm$ 0.02 \\
 & SPSA & \underline{\textbf{90.66 $\pm$ 0.30}} & 89.84 $\pm$ 0.31 & 89.82 $\pm$ 0.31 & 72.35 $\pm$ 0.46 & 6.51 $\pm$ 0.01 & 6.11 $\pm$ 0.00 & \underline{\textbf{5.17 $\pm$ 0.01}} & \textbf{5.31 $\pm$ 0.01} & 0.28 $\pm$ 0.09 & \underline{0.00 $\pm$ 0.00} & 0.16 $\pm$ 0.05 & 0.18 $\pm$ 0.03 \\
\cmidrule(lr){1-14}
\multirow{4}{*}{CW} & Normal & 89.81 $\pm$ 0.31 & 89.91 $\pm$ 0.31 & 89.76 $\pm$ 0.31 & \underline{91.80 $\pm$ 0.28} & 5.71 $\pm$ 0.01 & \textbf{5.62 $\pm$ 0.01} & \underline{5.60 $\pm$ 0.01} & 5.83 $\pm$ 0.01 & \textbf{0.19 $\pm$ 0.06} & \textbf{0.10 $\pm$ 0.05} & \textbf{0.14 $\pm$ 0.04} & \underline{\textbf{0.07 $\pm$ 0.01}} \\
 & FGSM & 89.79 $\pm$ 0.31 & \textbf{89.94 $\pm$ 0.31} & \textbf{90.08 $\pm$ 0.31} & \underline{\textbf{93.87 $\pm$ 0.25}} & \underline{\textbf{5.22 $\pm$ 0.02}} & 6.34 $\pm$ 0.01 & 5.99 $\pm$ 0.02 & 6.24 $\pm$ 0.02 & 0.23 $\pm$ 0.07 & 0.90 $\pm$ 0.28 & 0.22 $\pm$ 0.07 & \underline{0.14 $\pm$ 0.04} \\
 & PGD & \textbf{90.06 $\pm$ 0.31} & 89.89 $\pm$ 0.31 & 90.03 $\pm$ 0.31 & \underline{92.76 $\pm$ 0.27} & 5.24 $\pm$ 0.01 & 7.05 $\pm$ 0.01 & \underline{5.21 $\pm$ 0.01} & 5.64 $\pm$ 0.02 & 0.35 $\pm$ 0.11 & 0.90 $\pm$ 0.28 & \underline{0.28 $\pm$ 0.09} & 0.36 $\pm$ 0.11 \\
 & SPSA & 89.72 $\pm$ 0.31 & \underline{89.87 $\pm$ 0.31} & 89.40 $\pm$ 0.32 & 71.50 $\pm$ 0.46 & 6.07 $\pm$ 0.01 & 5.82 $\pm$ 0.01 & \underline{\textbf{4.97 $\pm$ 0.01}} & \textbf{5.00 $\pm$ 0.01} & \underline{0.19 $\pm$ 0.06} & 0.90 $\pm$ 0.29 & 0.22 $\pm$ 0.07 & 0.23 $\pm$ 0.07 \\
\cmidrule(lr){1-14}
\multirow{4}{*}{Clean} & Normal & 89.86 $\pm$ 0.31 & \textbf{90.48 $\pm$ 0.30} & 90.16 $\pm$ 0.31 & \underline{\textbf{93.91 $\pm$ 0.25}} & 2.74 $\pm$ 0.01 & 4.58 $\pm$ 0.01 & \underline{2.64 $\pm$ 0.01} & 3.23 $\pm$ 0.02 & 0.03 $\pm$ 0.01 & 0.05 $\pm$ 0.01 & \underline{0.01 $\pm$ 0.00} & 0.05 $\pm$ 0.01 \\
 & FGSM & 90.13 $\pm$ 0.31 & 90.22 $\pm$ 0.31 & 90.19 $\pm$ 0.31 & \underline{93.71 $\pm$ 0.25} & 2.73 $\pm$ 0.01 & 4.58 $\pm$ 0.01 & \underline{\textbf{2.57 $\pm$ 0.01}} & \textbf{3.16 $\pm$ 0.02} & \textbf{0.03 $\pm$ 0.01} & 0.06 $\pm$ 0.02 & \underline{0.01 $\pm$ 0.00} & 0.05 $\pm$ 0.01 \\
 & PGD & 89.69 $\pm$ 0.31 & 90.18 $\pm$ 0.31 & 90.28 $\pm$ 0.30 & \underline{93.16 $\pm$ 0.26} & \textbf{2.71 $\pm$ 0.01} & 4.34 $\pm$ 0.01 & \underline{2.60 $\pm$ 0.01} & 3.17 $\pm$ 0.02 & 0.03 $\pm$ 0.01 & 0.07 $\pm$ 0.02 & \underline{\textbf{0.00 $\pm$ 0.00}} & \textbf{0.04 $\pm$ 0.01} \\
 & SPSA & \textbf{90.24 $\pm$ 0.31} & 90.12 $\pm$ 0.31 & \textbf{90.41 $\pm$ 0.30} & \underline{93.72 $\pm$ 0.25} & 2.75 $\pm$ 0.01 & \textbf{4.29 $\pm$ 0.01} & \underline{2.61 $\pm$ 0.01} & 3.19 $\pm$ 0.02 & 0.04 $\pm$ 0.01 & \textbf{0.03 $\pm$ 0.01} & \underline{0.01 $\pm$ 0.00} & 0.05 $\pm$ 0.01 \\
\cmidrule(lr){1-14}
\label{RQ1TissueMNIST}
\end{tabular}
}
\end{table}
\begin{table}
\centering
\caption{\textbf{RQ2:} Mean and Standard Deviation of Coverage, Size, and SSCV for \textbf{TissueMNIST}}
\resizebox{0.9\textwidth}{!}{
\begin{tabular}{ll *{4}{c} *{4}{c} *{4}{c}}
\cmidrule(lr){1-14}
\makecell{Test} & \makecell{Defensive} &
\multicolumn{4}{c}{\makecell{Coverage (\%)}} &
\multicolumn{4}{c}{Size} &
\multicolumn{4}{c}{SSCV} \\
\cmidrule(lr){3-6} \cmidrule(lr){7-10} \cmidrule(lr){11-14}
Attack & Model & APS & RSCP & VRCP-I & VRCP-C & APS & RSCP & VRCP-I & VRCP-C & APS & RSCP & VRCP-I & VRCP-C \\
\cmidrule(lr){1-14}
\multirow{4}{*}{FGSM} & Normal & 90.13 $\pm$ 0.31 & \underline{96.51 $\pm$ 0.19} & 79.57 $\pm$ 0.41 & 91.87 $\pm$ 0.28 & \textbf{5.92 $\pm$ 0.00} & \textbf{6.94 $\pm$ 0.00} & \underline{5.45 $\pm$ 0.01} & \textbf{6.01 $\pm$ 0.00} & \underline{\textbf{0.00 $\pm$ 0.00}} & \textbf{0.07 $\pm$ 0.00} & 0.76 $\pm$ 0.33 & \textbf{0.02 $\pm$ 0.00} \\
 & FGSM & 93.68 $\pm$ 0.25 & \underline{98.84 $\pm$ 0.11} & \textbf{92.43 $\pm$ 0.27} & 95.93 $\pm$ 0.20 & 6.49 $\pm$ 0.01 & 7.67 $\pm$ 0.01 & \underline{6.12 $\pm$ 0.01} & 7.02 $\pm$ 0.01 & 0.10 $\pm$ 0.02 & 0.10 $\pm$ 0.01 & 0.10 $\pm$ 0.02 & \underline{0.06 $\pm$ 0.01} \\
 & PGD & \textbf{95.64 $\pm$ 0.21} & \underline{\textbf{99.74 $\pm$ 0.05}} & 92.12 $\pm$ 0.28 & \textbf{97.94 $\pm$ 0.15} & 6.41 $\pm$ 0.01 & 7.83 $\pm$ 0.00 & \underline{5.77 $\pm$ 0.01} & 6.91 $\pm$ 0.01 & 0.06 $\pm$ 0.03 & 0.10 $\pm$ 0.00 & \underline{\textbf{0.05 $\pm$ 0.01}} & 0.08 $\pm$ 0.02 \\
 & SPSA & 91.00 $\pm$ 0.29 & \underline{99.57 $\pm$ 0.07} & 70.72 $\pm$ 0.47 & 84.99 $\pm$ 0.37 & 6.55 $\pm$ 0.01 & 7.71 $\pm$ 0.01 & \underline{\textbf{5.08 $\pm$ 0.01}} & 6.29 $\pm$ 0.01 & 0.28 $\pm$ 0.09 & \underline{0.10 $\pm$ 0.00} & 0.24 $\pm$ 0.04 & 0.26 $\pm$ 0.08 \\
\cmidrule(lr){1-14}
\multirow{4}{*}{PGD} & Normal & 90.12 $\pm$ 0.31 & \underline{96.24 $\pm$ 0.20} & 79.16 $\pm$ 0.42 & 91.66 $\pm$ 0.28 & \textbf{5.91 $\pm$ 0.00} & \textbf{6.93 $\pm$ 0.00} & \underline{5.44 $\pm$ 0.01} & \textbf{5.99 $\pm$ 0.00} & \underline{\textbf{0.00 $\pm$ 0.00}} & \textbf{0.06 $\pm$ 0.00} & 0.90 $\pm$ 0.40 & \textbf{0.02 $\pm$ 0.00} \\
 & FGSM & 93.80 $\pm$ 0.25 & \underline{98.91 $\pm$ 0.11} & \textbf{92.75 $\pm$ 0.27} & 95.66 $\pm$ 0.21 & 6.51 $\pm$ 0.01 & 7.67 $\pm$ 0.01 & \underline{6.13 $\pm$ 0.01} & 7.03 $\pm$ 0.01 & 0.10 $\pm$ 0.02 & 0.10 $\pm$ 0.01 & 0.07 $\pm$ 0.02 & \underline{0.06 $\pm$ 0.01} \\
 & PGD & \textbf{95.70 $\pm$ 0.21} & \underline{\textbf{99.71 $\pm$ 0.05}} & 92.12 $\pm$ 0.28 & \textbf{98.00 $\pm$ 0.14} & 6.41 $\pm$ 0.01 & 7.83 $\pm$ 0.00 & \underline{5.77 $\pm$ 0.01} & 6.91 $\pm$ 0.01 & 0.06 $\pm$ 0.01 & 0.10 $\pm$ 0.00 & \underline{\textbf{0.05 $\pm$ 0.01}} & 0.08 $\pm$ 0.03 \\
 & SPSA & 90.79 $\pm$ 0.30 & \underline{99.65 $\pm$ 0.06} & 70.98 $\pm$ 0.47 & 84.63 $\pm$ 0.37 & 6.54 $\pm$ 0.01 & 7.71 $\pm$ 0.01 & \underline{\textbf{5.08 $\pm$ 0.01}} & 6.29 $\pm$ 0.01 & 0.15 $\pm$ 0.05 & 0.10 $\pm$ 0.00 & 0.21 $\pm$ 0.04 & \underline{0.05 $\pm$ 0.01} \\
\cmidrule(lr){1-14}
\multirow{4}{*}{SPSA} & Normal & 90.13 $\pm$ 0.31 & \underline{96.43 $\pm$ 0.19} & 79.65 $\pm$ 0.41 & 92.05 $\pm$ 0.28 & \textbf{5.92 $\pm$ 0.00} & \textbf{6.94 $\pm$ 0.00} & \underline{5.45 $\pm$ 0.01} & \textbf{6.01 $\pm$ 0.00} & \underline{\textbf{0.00 $\pm$ 0.00}} & \textbf{0.06 $\pm$ 0.00} & 0.50 $\pm$ 0.20 & \textbf{0.02 $\pm$ 0.00} \\
 & FGSM & 93.78 $\pm$ 0.25 & \underline{98.68 $\pm$ 0.12} & \textbf{92.71 $\pm$ 0.27} & 95.75 $\pm$ 0.21 & 6.49 $\pm$ 0.01 & 7.67 $\pm$ 0.01 & \underline{6.11 $\pm$ 0.01} & 7.02 $\pm$ 0.01 & 0.23 $\pm$ 0.07 & 0.10 $\pm$ 0.01 & 0.10 $\pm$ 0.03 & \underline{0.06 $\pm$ 0.02} \\
 & PGD & \textbf{95.70 $\pm$ 0.21} & \underline{\textbf{99.75 $\pm$ 0.05}} & 92.27 $\pm$ 0.27 & \textbf{97.99 $\pm$ 0.14} & 6.41 $\pm$ 0.01 & 7.83 $\pm$ 0.00 & \underline{5.77 $\pm$ 0.01} & 6.91 $\pm$ 0.01 & 0.06 $\pm$ 0.01 & 0.10 $\pm$ 0.00 & \underline{\textbf{0.03 $\pm$ 0.00}} & 0.08 $\pm$ 0.03 \\
 & SPSA & 90.69 $\pm$ 0.30 & \underline{99.62 $\pm$ 0.06} & 71.03 $\pm$ 0.47 & 84.68 $\pm$ 0.37 & 6.55 $\pm$ 0.01 & 7.71 $\pm$ 0.01 & \underline{\textbf{5.07 $\pm$ 0.01}} & 6.27 $\pm$ 0.01 & 0.30 $\pm$ 0.10 & \underline{0.10 $\pm$ 0.00} & 0.26 $\pm$ 0.04 & 0.26 $\pm$ 0.08 \\
\cmidrule(lr){1-14}
\multirow{4}{*}{CW} & Normal & 89.92 $\pm$ 0.31 & \underline{96.41 $\pm$ 0.19} & 78.88 $\pm$ 0.42 & 91.75 $\pm$ 0.28 & \textbf{5.71 $\pm$ 0.01} & \textbf{6.75 $\pm$ 0.01} & \underline{5.18 $\pm$ 0.01} & 5.83 $\pm$ 0.01 & 0.21 $\pm$ 0.06 & \underline{0.10 $\pm$ 0.01} & \textbf{0.14 $\pm$ 0.03} & 0.30 $\pm$ 0.09 \\
 & FGSM & 92.52 $\pm$ 0.27 & \underline{98.69 $\pm$ 0.12} & \textbf{91.24 $\pm$ 0.29} & 94.93 $\pm$ 0.23 & 5.93 $\pm$ 0.02 & 7.16 $\pm$ 0.01 & \underline{5.57 $\pm$ 0.02} & 6.45 $\pm$ 0.02 & 0.23 $\pm$ 0.07 & \underline{\textbf{0.09 $\pm$ 0.01}} & 0.22 $\pm$ 0.07 & \textbf{0.12 $\pm$ 0.02} \\
 & PGD & \textbf{94.19 $\pm$ 0.24} & \underline{\textbf{99.58 $\pm$ 0.07}} & 90.06 $\pm$ 0.31 & \textbf{96.89 $\pm$ 0.18} & 5.89 $\pm$ 0.02 & 7.40 $\pm$ 0.01 & \underline{5.28 $\pm$ 0.01} & 6.38 $\pm$ 0.01 & 0.29 $\pm$ 0.09 & \underline{0.10 $\pm$ 0.02} & 0.43 $\pm$ 0.14 & 0.23 $\pm$ 0.06 \\
 & SPSA & 89.57 $\pm$ 0.31 & \underline{99.40 $\pm$ 0.08} & 68.99 $\pm$ 0.48 & 83.49 $\pm$ 0.38 & 6.07 $\pm$ 0.01 & 7.29 $\pm$ 0.01 & \underline{\textbf{4.68 $\pm$ 0.01}} & \textbf{5.82 $\pm$ 0.01} & \underline{\textbf{0.18 $\pm$ 0.05}} & 0.28 $\pm$ 0.07 & 0.24 $\pm$ 0.07 & 0.21 $\pm$ 0.05 \\
\cmidrule(lr){1-14}
\multirow{4}{*}{Clean} & Normal & \textbf{98.31 $\pm$ 0.13} & \underline{\textbf{99.56 $\pm$ 0.07}} & \textbf{95.25 $\pm$ 0.22} & \textbf{98.83 $\pm$ 0.11} & 4.47 $\pm$ 0.02 & 5.63 $\pm$ 0.02 & \underline{3.57 $\pm$ 0.02} & 4.76 $\pm$ 0.02 & 0.09 $\pm$ 0.00 & 0.10 $\pm$ 0.00 & \underline{0.06 $\pm$ 0.01} & 0.09 $\pm$ 0.00 \\
 & FGSM & 89.93 $\pm$ 0.31 & \underline{98.02 $\pm$ 0.14} & 86.85 $\pm$ 0.35 & 93.61 $\pm$ 0.25 & \textbf{2.70 $\pm$ 0.01} & \textbf{4.31 $\pm$ 0.02} & \underline{2.43 $\pm$ 0.01} & 3.20 $\pm$ 0.02 & \underline{0.03 $\pm$ 0.01} & \textbf{0.08 $\pm$ 0.00} & \textbf{0.06 $\pm$ 0.01} & \textbf{0.04 $\pm$ 0.00} \\
 & PGD & 89.93 $\pm$ 0.31 & \underline{98.90 $\pm$ 0.11} & 85.51 $\pm$ 0.36 & 93.59 $\pm$ 0.25 & 2.77 $\pm$ 0.01 & 4.90 $\pm$ 0.02 & \underline{\textbf{2.34 $\pm$ 0.01}} & 3.23 $\pm$ 0.02 & \underline{\textbf{0.03 $\pm$ 0.01}} & 0.09 $\pm$ 0.01 & 0.08 $\pm$ 0.02 & 0.05 $\pm$ 0.01 \\
 & SPSA & 94.46 $\pm$ 0.24 & \underline{98.87 $\pm$ 0.11} & 86.01 $\pm$ 0.36 & 93.55 $\pm$ 0.25 & 3.33 $\pm$ 0.02 & 4.78 $\pm$ 0.02 & \underline{2.34 $\pm$ 0.01} & \textbf{3.16 $\pm$ 0.02} & 0.06 $\pm$ 0.01 & 0.09 $\pm$ 0.00 & 0.07 $\pm$ 0.01 & \underline{0.04 $\pm$ 0.01} \\
\cmidrule(lr){1-14}
\label{RQ2TissueMNIST}
\end{tabular}
}
\end{table}

Finally, we evaluate the performance on the TissueMNIST dataset, as shown in Tables \ref{RQ1TissueMNIST} and \ref{RQ2TissueMNIST}. For RQ1, with the exception of VRCP-C when the SPSA defensive model faces other perturbed attacks, all CP methods achieve effective coverage for every defensive model under each attack. In most cases, VRCP-I produces the smallest prediction set sizes, VRCP-C reaches the highest coverage, and APS demonstrates superior SSCV. For RQ2, most CP methods require a larger \(q\) value compared to the RQ1 condition, and notably, the CP coverages under RQ2 can even reach 99\%. Compared with the other CP methods, APS maintains a more stable coverage around 90\% with better SSCV. RSCP is more conservative, achieving higher coverage at the cost of larger prediction set sizes, while VRCP-C strikes a balance between the two. Although VRCP-I consistently produces the smallest set sizes, its coverage becomes unstable for some defensive models—especially for the SPSA and Normal models.
\subsubsection{Results for RQ3} \label{subsec:rq3_experiment}
In the third experiment, an unknown and potentially adversarial attack is applied to the test data. The prediction set size corresponds to the attacker's payoff and the defender's cost, framing the interaction as a zero-sum game. To develop an optimal solution for the defender in a strategic manner, we preserve an additional evaluation set, which is exchangeable with both calibration and test sets. We use the prediction set size evaluated on this evaluation set as the estimated payoff matrix and derive the Nash equilibrium to determine the defender's optimal strategy. This strategy is then applied during the test phase. Due to the unstable coverage exhibited by the VRCP-based CP methods on certain defensive models or datasets in RQ1 and RQ2, we restrict our experiments in RQ3 to only the APS and RSCP methods.

In addition to the defensive models \( f_j \) and the normal model \( f_0 \), we implement a Maximum classifier \( f_{\text{max}} \) and a Minimum classifier \( f_{\text{min}} \). These classifiers aggregate the logits from multiple defensive models by taking the element-wise maximum and minimum across all models, respectively. Formally, for a given input \( x_i \) and class \( y \), they are defined as:
\begin{align*}
    f_{\text{max}}(x_i)[y] &= \max_{j} f_j(x_i)[y], \\
    f_{\text{min}}(x_i)[y] &= \min_{j} f_j(x_i)[y],
\end{align*}
where \( f_j(x_i)[y] \) denotes the predicted probability of the \( j \)-th defensive model for class \( y \) given input \( x_i \).

The payoff matrix in this scenario has a shape of \( 6 \times 5 \), accounting for the five defensive strategies (including \( f_{\text{max}} \) and \( f_{\text{min}} \)) and the five attack types. Furthermore, the payoff matrix describes the size of the prediction set corresponding to different defense methods under various attack methods. We divided the prediction set into two parts, one for evaluation and the other for testing. The evaluation payoff matrix (where the size of the prediction set is considered as the payoff object, with the attacker aiming to maximize it and the defender aiming to minimize it) is designed to calculate the optimal attack strategy and optimal defense strategy at the Nash equilibrium. To verify the effectiveness of our method, we computed the payoff matrix on both the evaluation set and the test set. Additionally, we incorporated a uniform strategy into the test set, where each defense model randomly selects and computes the size of the prediction set in equal proportions. All box plots illustrate the size of the prediction set for various defense models under the optimal attack strategy found on the evaluation set.

\begin{figure}[htbp]
    \centering
    \begin{minipage}{0.45\textwidth}
        \centering
        \includegraphics[width=\textwidth]{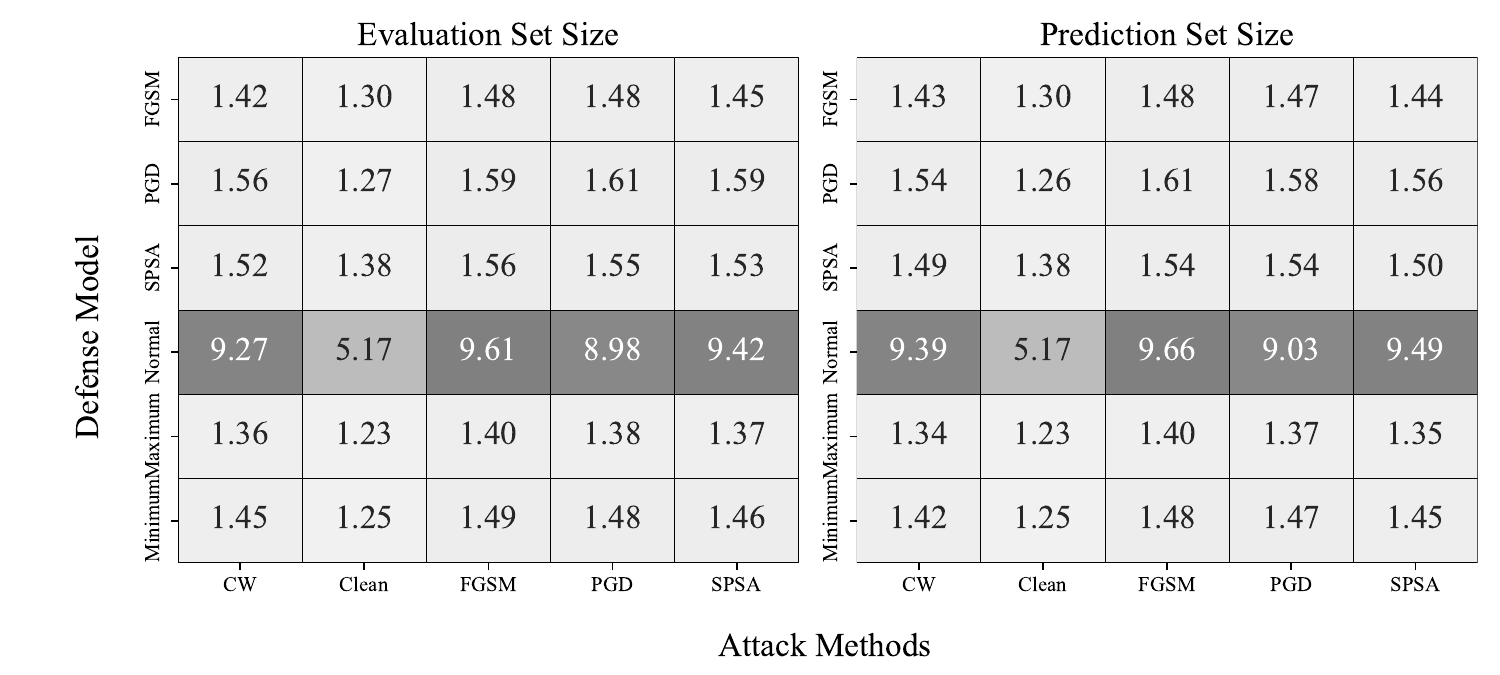}
        \caption{\textbf{RQ3:} Estimated and True Payoff Matrices for OrganAMNIST (APS)}
        \label{fig:OrganAMNIST_APS}
    \end{minipage}
    \hfill
    \begin{minipage}{0.45\textwidth}
        \centering
        \includegraphics[width=\textwidth]{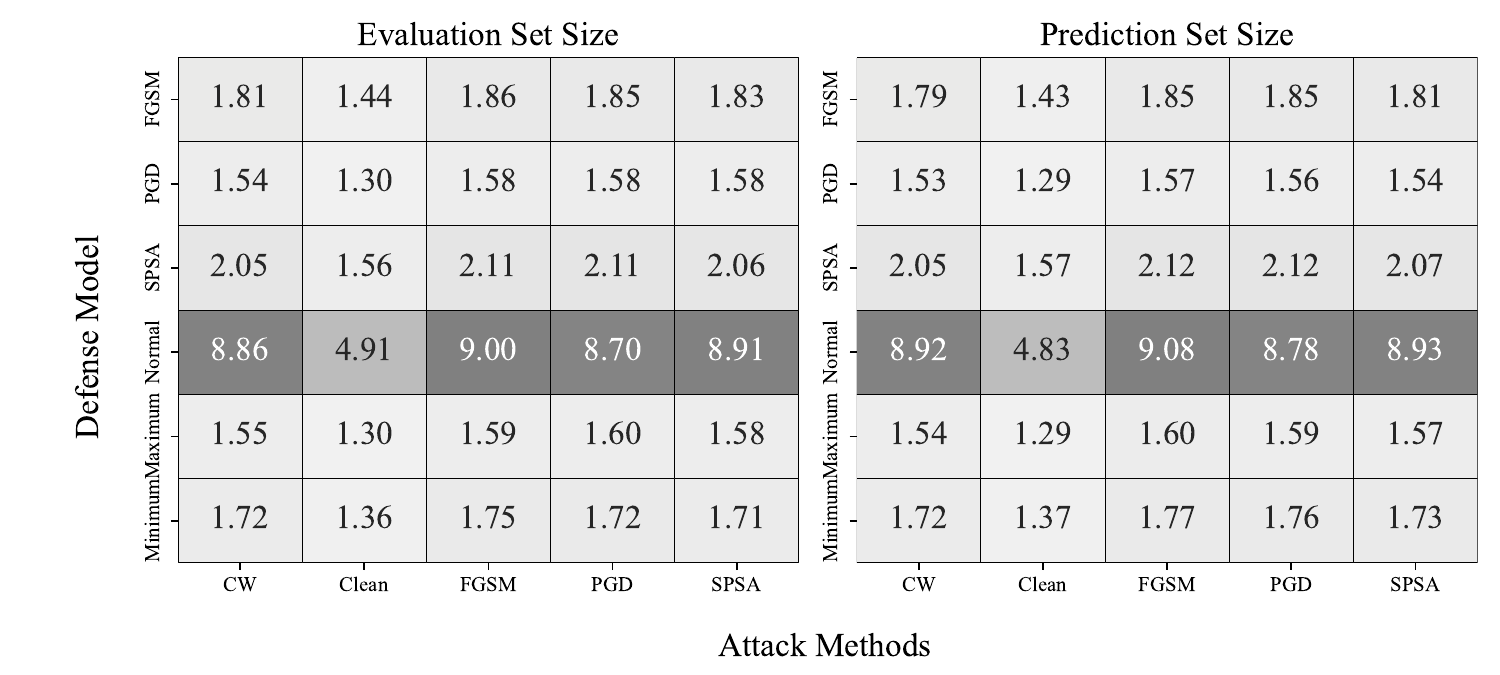}
        \caption{\textbf{RQ3:} Estimated and True Payoff Matrices for OrganAMNIST (RSCP)}
        \label{fig:OrganAMNIST_RSCP}
    \end{minipage}
\end{figure}

\begin{figure}[htbp]
    \centering
    \begin{minipage}{0.45\textwidth}
        \centering
        \includegraphics[width=\textwidth]{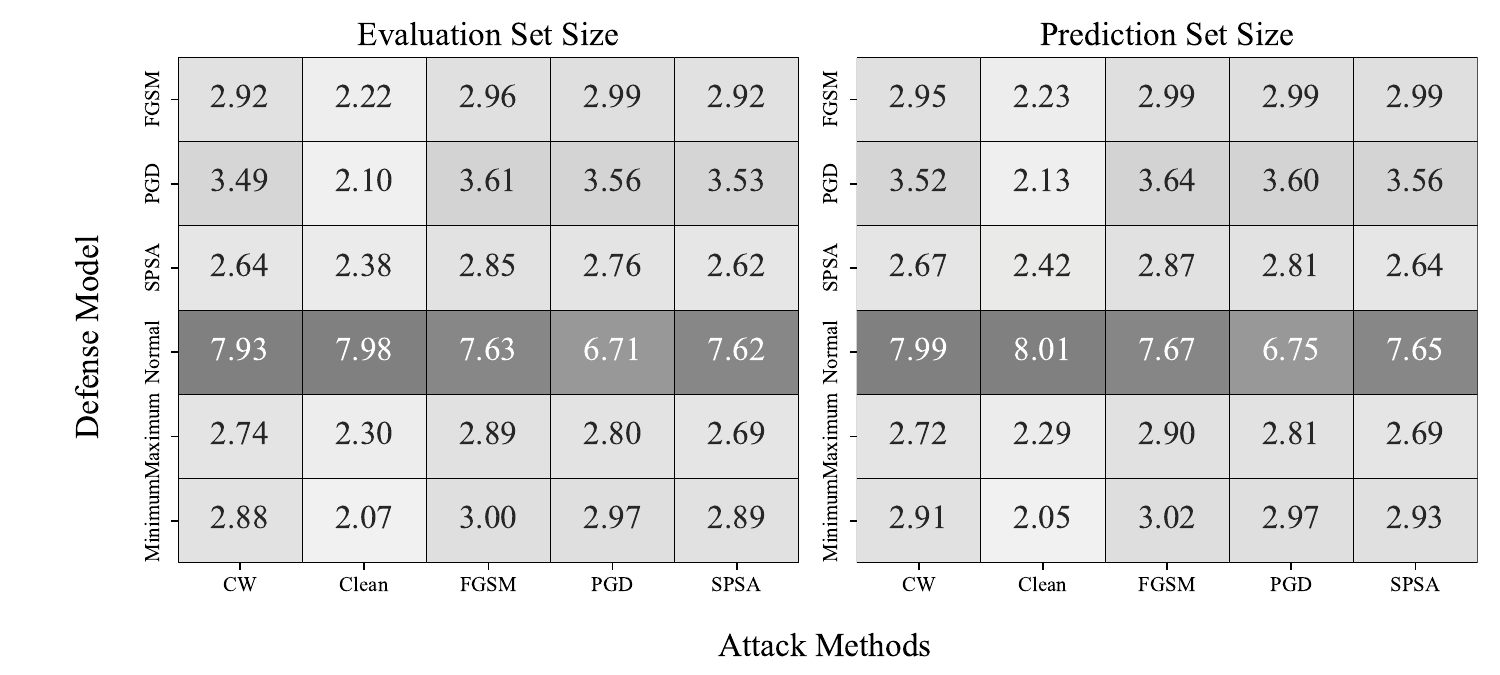}
        \caption{\textbf{RQ3:} Estimated and True Payoff Matrices for PathMNIST (APS)}
        \label{fig:PathMNIST_APS}
    \end{minipage}
    \hfill
    \begin{minipage}{0.45\textwidth}
        \centering
        \includegraphics[width=\textwidth]{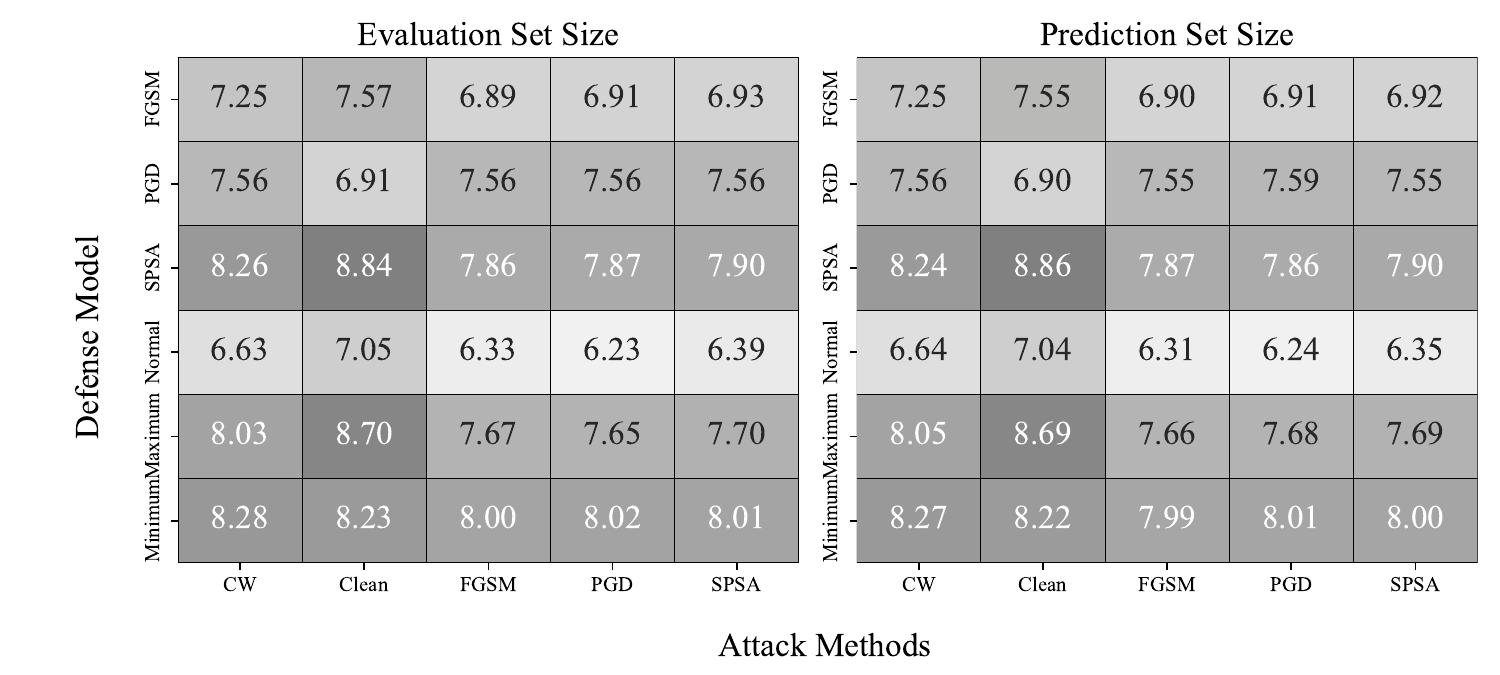}
        \caption{\textbf{RQ3:} Estimated and True Payoff Matrices for PathMNIST (RSCP)}
        \label{fig:PathMNIST_RSCP}
    \end{minipage}
\end{figure}

\begin{figure}[htbp]
    \centering
    \begin{minipage}{0.45\textwidth}
        \centering
        \includegraphics[width=\textwidth]{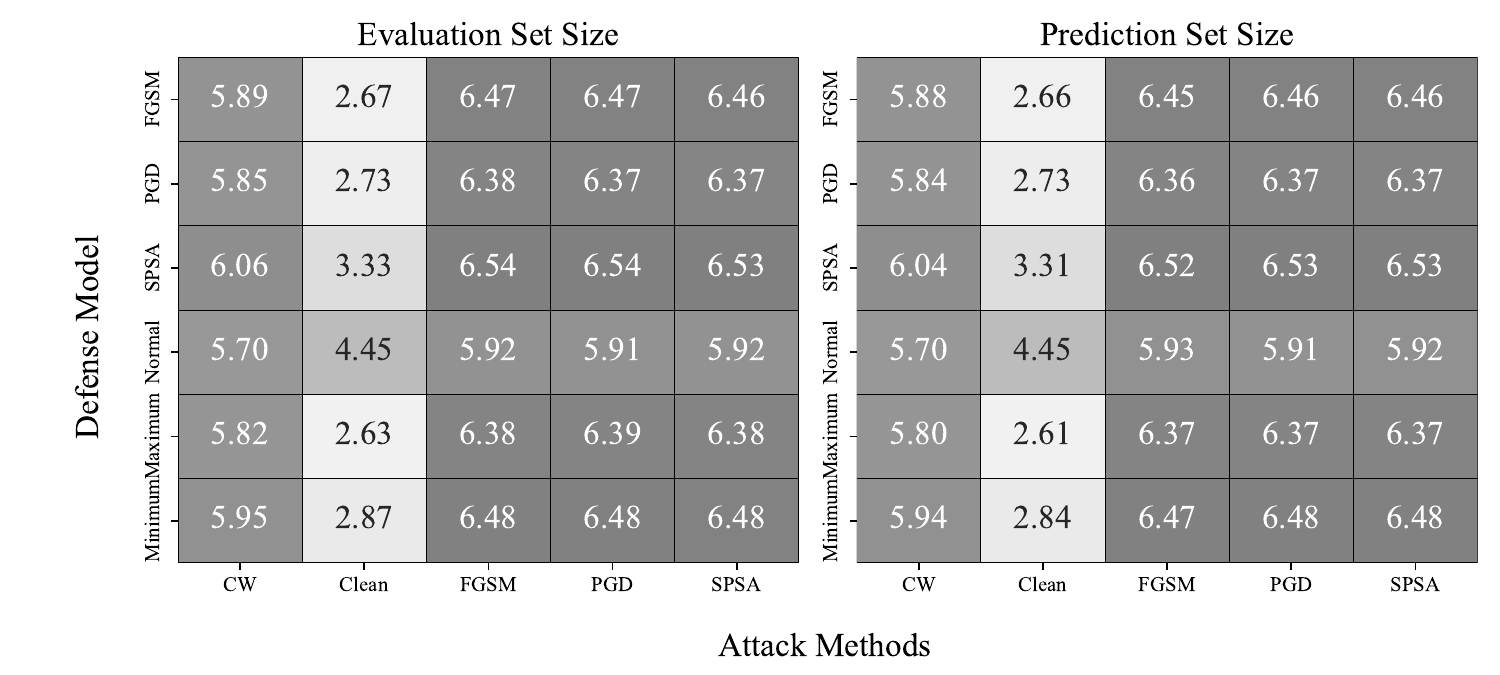}
        \caption{\textbf{RQ3:} Estimated and True Payoff Matrices for TissueMNIST (APS)}
        \label{fig:TissueMNIST_APS}
    \end{minipage}
    \hfill
    \begin{minipage}{0.45\textwidth}
        \centering
        \includegraphics[width=\textwidth]{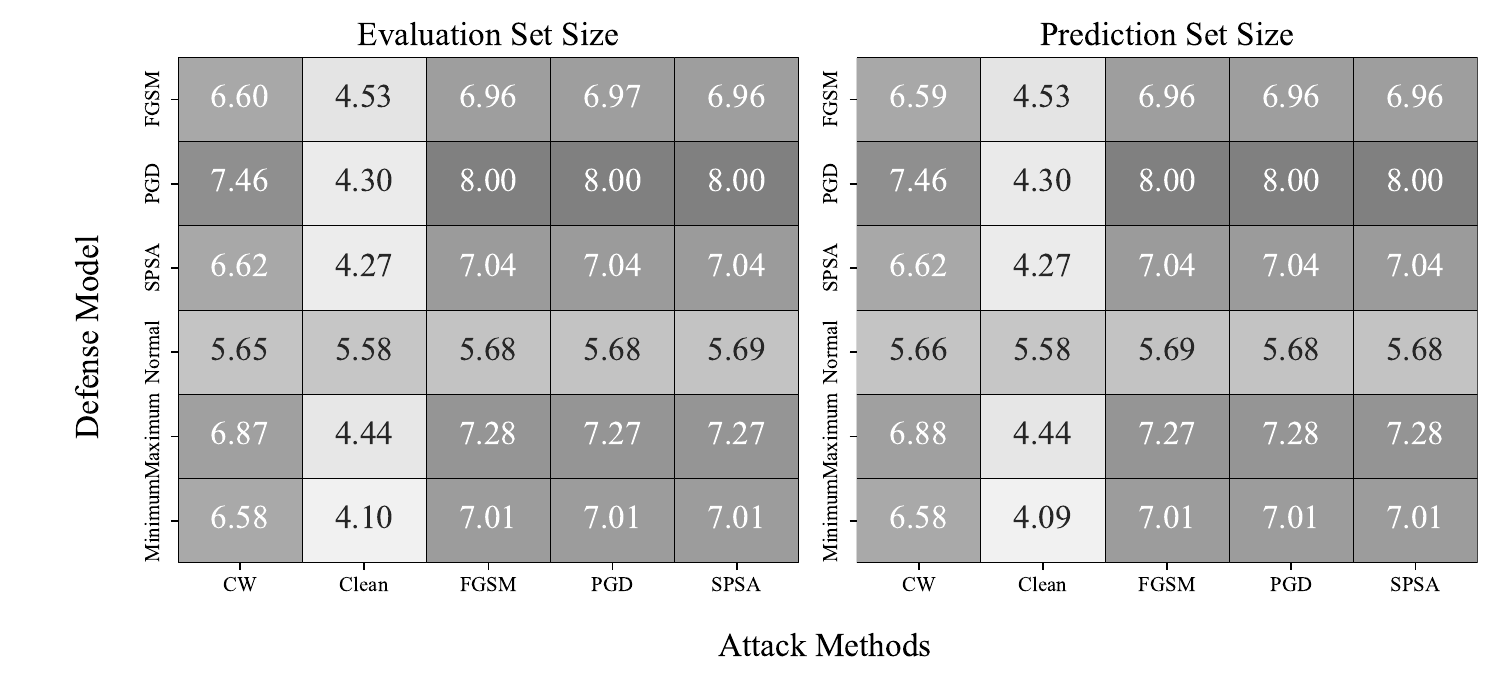}
        \caption{\textbf{RQ3:} Estimated and True Payoff Matrices for TissueMNIST (RSCP)}
        \label{fig:TissueMNIST_RSCP}
    \end{minipage}
\end{figure}

\begin{figure}[htbp]
    \centering
    \begin{subfigure}[t]{0.45\textwidth}
        \centering
        \includegraphics[width=\linewidth]{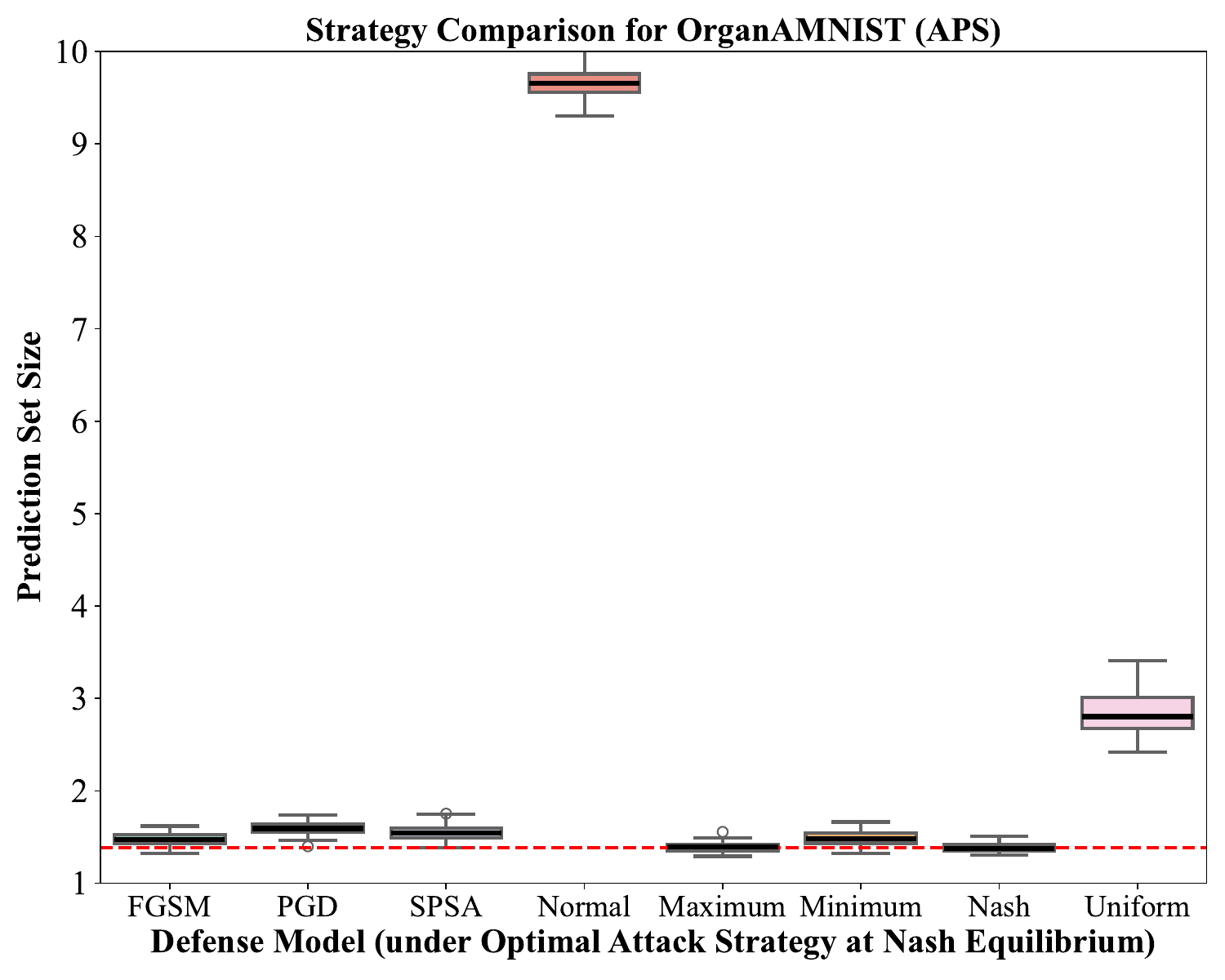}
        \label{fig:OrganAMNIST_APS_nash}
    \end{subfigure}
    \hfill
    \begin{subfigure}[t]{0.45\textwidth}
        \centering
        \includegraphics[width=\linewidth]{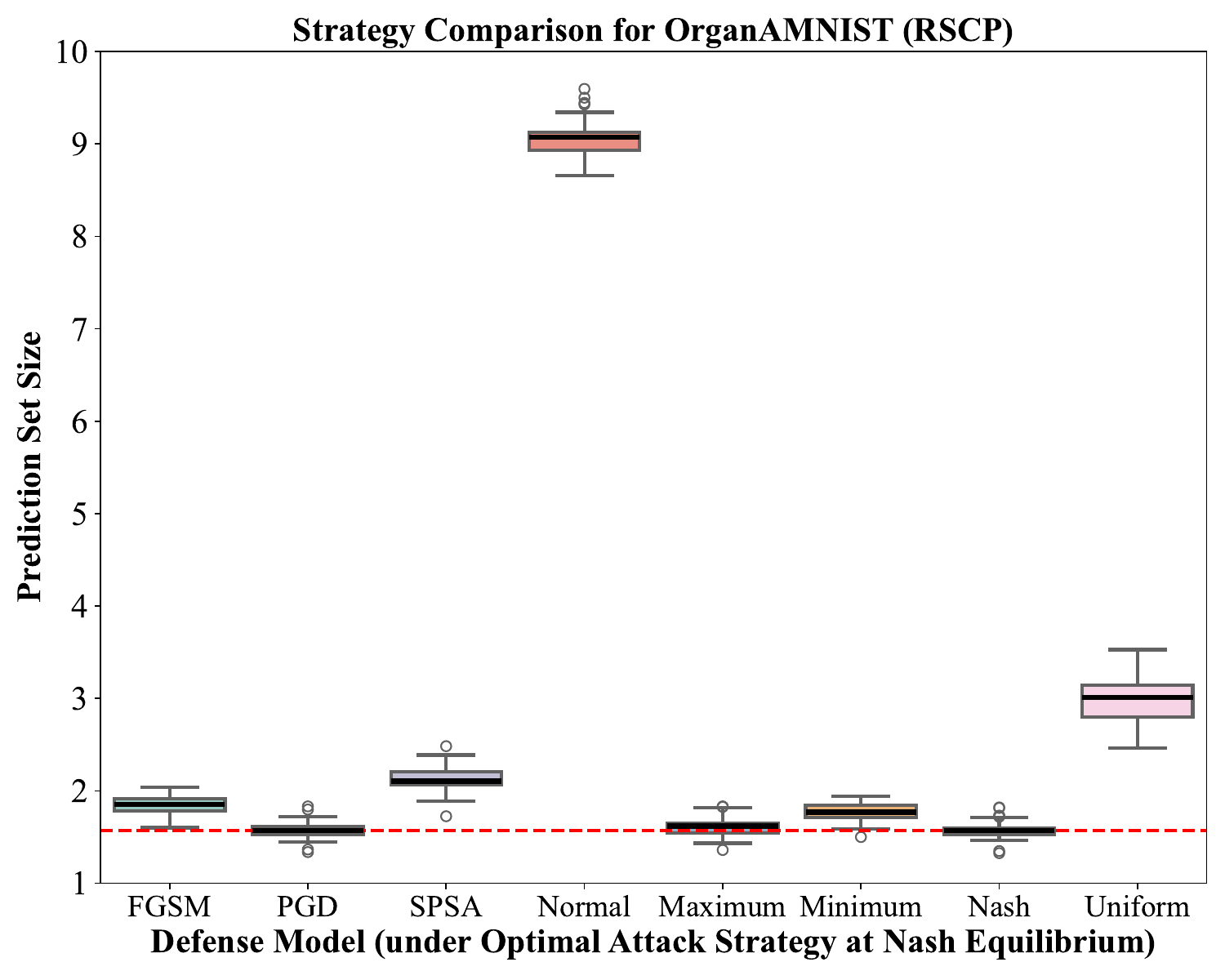}
        \label{fig:OrganAMNIST_RSCP_nash}
    \end{subfigure}

    \vspace{10pt} %

    \begin{subfigure}[t]{0.45\textwidth}
        \centering
        \includegraphics[width=\linewidth]{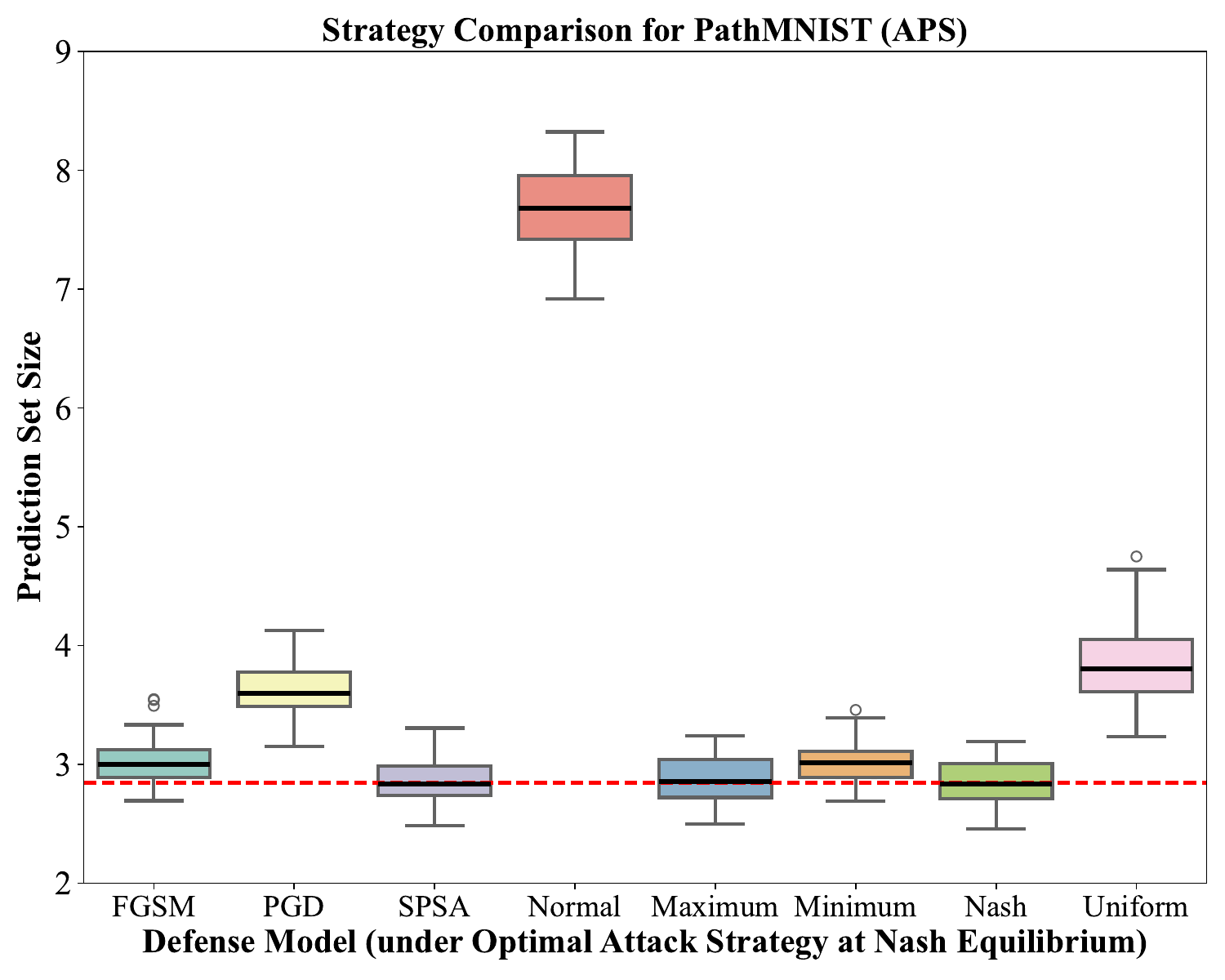}
        \label{fig:PathMNIST_APS_nash}
    \end{subfigure}
    \hfill
    \begin{subfigure}[t]{0.45\textwidth}
        \centering
        \includegraphics[width=\linewidth]{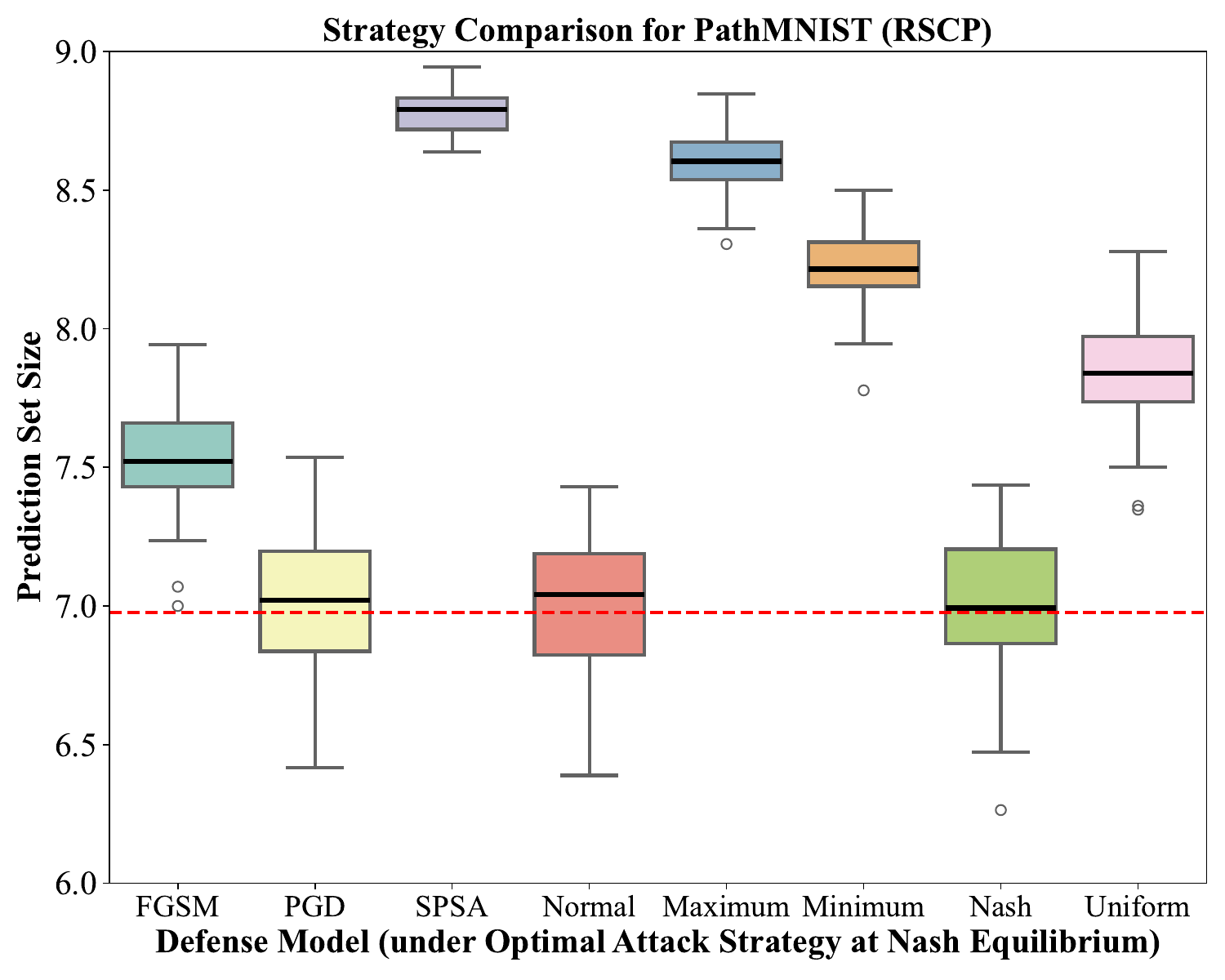}
        \label{fig:PathMNIST_RSCP_nash}
    \end{subfigure}
        \vspace{10pt} %

    \begin{subfigure}[t]{0.45\textwidth}
        \centering
        \includegraphics[width=\linewidth]{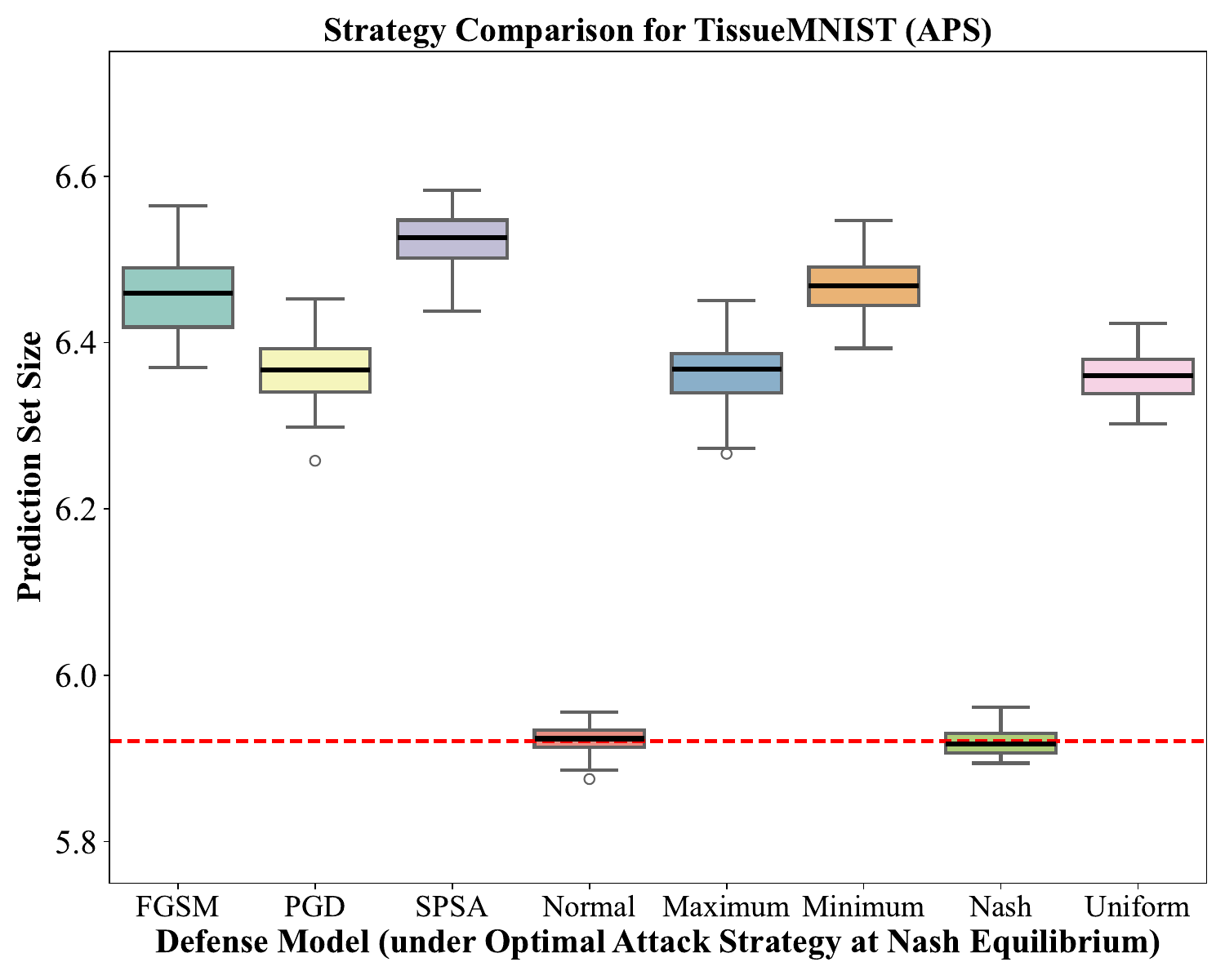}
        \label{fig:TissueMNIST_APS_nash}
    \end{subfigure}
    \hfill
    \begin{subfigure}[t]{0.45\textwidth}
        \centering
        \includegraphics[width=\linewidth]{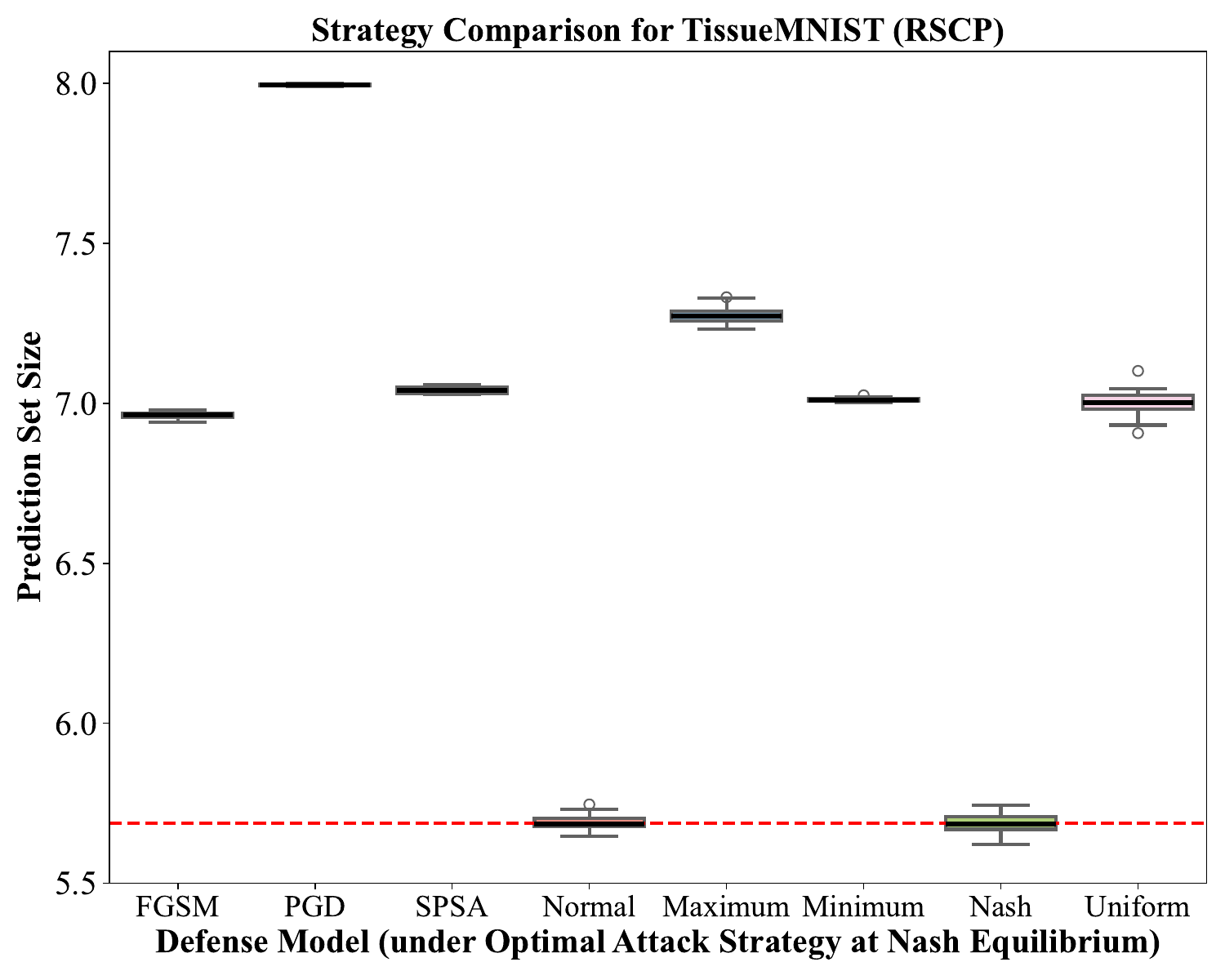}
        \label{fig:TissueMNIST_RSCP_nash}
    \end{subfigure}

    \caption{\textbf{RQ3:} Comparison of strategies across different datasets and CP methods.}
    \label{fig:StrategyComparison}
\end{figure}

From the reward matrix depicted in Figure \ref{fig:OrganAMNIST_APS}-\ref{fig:TissueMNIST_RSCP}, it can be discerned that, in comparison to the APS method, the RSCP generally yields larger prediction sets. This observation is supported by the findings of RQ1 and RQ2, which attribute this to the higher coverage achieved by RSCP through the utilization of more stable classification points (CPs). Furthermore, when applying the RSCP method on the PathMNIST dataset, the constructed Nash equilibrium scenario involves both the defender and the attacker employing multiple defense and attack models to reach a zero-sum game process. Consequently, in the box plot shown in Figure \ref{fig:StrategyComparison}, the Nash defense model exhibits a smaller prediction set compared to other methods. In contrast, for other datasets, achieving the Nash equilibrium only requires a single strategy, which explains why the Nash defense model in Figure \ref{fig:StrategyComparison} appears nearly identical to one of the defense models. Additionally, except for the PathMNIST dataset where the two CP methods exhibit some inconsistency, the performance differences among various defense models are relatively minor across other datasets.

\subsection{Discussion}
The experiments validate that adversarial training tailored to specific attack types enhances the robustness of conformal prediction sets. In scenarios with unknown attacks, leveraging a validation set to optimally combine defenses ensures maintained coverage with efficient prediction sizes. This aligns with findings from Liu et al. \cite{liu2024pitfalls}, which emphasize the importance of considering CP efficiency during adversarial training.

However, there are limitations. The premise in RQ1, which assumes that attack types are limited to known types, may not hold true in all real-world scenarios. To address RQ2 and RQ3, we have simplified strategy selection in our current work by restricting both attackers and defenders to a discrete action space. Furthermore, for RQ3, we allow the use of weighted combinations of defense models to form mixed strategies, thereby eliminating the need for exhaustive grid search for optimal weights. This design alleviates computational overhead to a certain extent, making the method more feasible for large-scale datasets or more sophisticated adversarial strategies.

Future work may explore dynamic adaptation to a broader range of attacks, extend the framework to other medical imaging datasets, and investigate the integration of uncertainty-reducing adversarial training methods \cite{colombo2020training, stutz2022learning, wang2025enhancing} to further enhance CP efficiency.

\section{Conclusion}
This study presents a robust framework for constructing conformal prediction sets on medical imaging datasets under adversarial attacks. By training specialized models for distinct attack types and employing strategic model selection and weighting, we achieve high coverage guarantees with minimal prediction set sizes. Crucially, our methodology integrates game-theoretic principles to formulate and identify optimal defensive strategies within a zero-sum game framework between the attacker and defender. This game-theoretic approach ensures that the defensive strategies are not only resilient against known and unknown adversarial perturbations but also strategically optimized to mitigate the most severe threats posed by adaptive adversaries.

In summary, the proposed methodology synergizes adversarial robustness with uncertainty quantification through a weighted combination of defensive models and the conformal prediction framework. By systematically calibrating and evaluating different weight configurations within a game-theoretic context, our approach guarantees that the final prediction sets are both reliable and efficient, maintaining high coverage rates while minimizing ambiguity in predictions. The incorporation of game-theoretic defensive strategies enhances the ability of the conformal prediction system to adaptively respond to diverse and evolving adversarial attacks, thereby positioning our methodology as a comprehensive and strategic solution for deploying deep learning models in adversarial environments. This dual focus on resilience against adversarial perturbations and the provision of meaningful uncertainty estimates significantly advances the reliability and security of AI systems in critical applications such as healthcare.

\end{document}